\documentclass{article} 
\usepackage{colm2024_conference}
\usepackage[utf8]{inputenc} 
\usepackage[T1]{fontenc}    
\usepackage{hyperref}       
\usepackage{url}            
\usepackage{booktabs}       
\usepackage{amsfonts}       
\usepackage{nicefrac}       
\usepackage{microtype}      
\usepackage{xcolor}         
\usepackage{lipsum}
\usepackage{graphicx}
\usepackage{wrapfig}
\usepackage[size=small]{caption}
\usepackage{mathtools}
\usepackage{amssymb}
\usepackage{amsthm}
\usepackage{soul}
\usepackage{tikz,lipsum,lmodern}
\usepackage[most]{tcolorbox}
\usepackage{cleveref}

\usepackage[algoruled,boxed,lined,noend]{algorithm2e}
\newtheorem{theorem}{Theorem}

\newtheorem{lemma}{Lemma}
\newtheorem{definition}{Definition}
\newcommand{\piref}{\pi_\text{ref}}

\newcommand{\x}{\ensuremath{\boldsymbol{x}}}

\newcommand{\vc}{\ensuremath{\boldsymbol{c}}}

\newcommand{\bbD}{\ensuremath{\mathbb{D}}}
\newcommand{\kl}{\ensuremath{\bbD_{\text{KL}}}}

\newcommand{\bs}{\mathbf{s}}
\newcommand{\ba}{\mathbf{a}}
\newcommand{\bx}{\mathbf{x}}
\newcommand{\by}{\mathbf{y}}

\usepackage{amsmath,amsfonts,bm}

\def\eqref#1{equation~\ref{#1}}

\def\1{\bm{1}}

\def\vc{{\bm{c}}}

\DeclareMathAlphabet{\mathsfit}{\encodingdefault}{\sfdefault}{m}{sl}
\SetMathAlphabet{\mathsfit}{bold}{\encodingdefault}{\sfdefault}{bx}{n}

\newcommand{\E}{\mathbb{E}}

\usepackage{amssymb}            
\usepackage{mathtools}          
\usepackage{mathrsfs}           
\usepackage{graphicx}           
\usepackage{subcaption}         
\usepackage[space]{grffile}     
\usepackage{url}                

\usepackage{microtype}
\usepackage{hyperref}
\usepackage{url}
\usepackage{booktabs}
\definecolor{darkblue}{rgb}{0, 0, 0.5}
\hypersetup{colorlinks=true, citecolor=darkblue, linkcolor=darkblue, urlcolor=darkblue}

\newcommand{\blfootnote}[1]{%
  \begingroup
  \renewcommand\thefootnote{}\footnote{#1}%
  \addtocounter{footnote}{-1}%
  \endgroup
}

\title{From $r$ to $Q^*$: Your Language Model is Secretly a Q-Function}

\author{Rafael Rafailov*  \\
Stanford University\\
\texttt{rafailov@stanford.edu} \\
\And
Joey Hejna*\\
Stanford University\\
\texttt{jhejna@stanford.edu} \\
\And
Ryan Park\\
Stanford University\\
\texttt{rypark@stanford.edu} \\
\And
Chelsea Finn\\
Stanford University\\
\texttt{cbfinn@stanford.edu} \\
}

\colmfinalcopy 
\begin{document}

\maketitle

\begin{abstract}
Reinforcement Learning From Human Feedback (RLHF) has been critical to the success of the latest generation of generative AI models. In response to the complex nature of the classical RLHF pipeline, direct alignment algorithms such as Direct Preference Optimization (DPO) have emerged as an alternative approach. Although DPO solves the same objective as the standard RLHF setup, there is a mismatch between the two approaches. Standard RLHF deploys reinforcement learning in a specific token-level MDP, while DPO is derived as a bandit problem in which the whole response of the model is treated as a single arm. In this work we rectify this difference. We theoretically show that we can derive DPO in the token-level MDP as a general inverse Q-learning algorithm, which satisfies the Bellman equation. Using our theoretical results, we provide three concrete empirical insights. First, we show that because of its token level interpretation, DPO is able to perform some type of credit assignment. Next, we prove that under the token level formulation, classical search-based algorithms, such as MCTS, which have recently been applied to the language generation space, are equivalent to likelihood-based search on a DPO policy. Empirically we show that a simple beam search yields meaningful improvement over the base DPO policy. Finally, we show how the choice of reference policy causes implicit rewards to decline during training. We conclude by discussing applications of our work, including information elicitation in multi-turn dialogue, reasoning, agentic applications and end-to-end training of multi-model systems.

\blfootnote{\textbf{*} Denotes equal contribution}
\end{abstract}

\section{Introduction}
Reinforcement Learning from Human Feedback (RLHF) has become the defacto method for aligning large language models (LLMs) with human intent due to its success in a wide range of applications from summarization \citep{stiennon2022learning} to instruction following \citep{ouyang2022training}. By learning a reward function from human-labeled comparisons, RLHF is able to capture complex objectives  that are in-describedable in practice.  Following the success of \citep{ziegler2020finetuning}, numerous works have considered new algorithms for training and sampling from large models in various domains using techniques from reinforcement learning (RL). In particular direct alignment methods, such as Direct Preference Optimization (DPO) \citep{rafailov2023direct} have gained traction in recent months because of their simplicity \citep{zhao2023slic, azar2023general}. Instead of learning a reward function and then using RL, direct alignment methods use the relationship between reward functions and policies in the contextual bandit setting to optimize both simultaneously. Similar ideas have since been applied to vision language \citep{zhao2023beyond} and image generation models \citep{lee2023aligning}. \looseness=-1

While such direct alignment methods are purported to work the same as classical RLHF approaches that use policy gradient algorithms like PPO \citep{schulman2017proximal}, fundamental differences remain. For instance, classical RLHF methods optimize token-level value functions with a sparse reward at the terminal state. DPO on the other hand, operates only in a contextual bandits setting, treating the entire response as a single arm. This is despite the fact that tokens are generated one at a time, and dense rewards are commonly known to be beneficial in the RL community. While direct alignment algorithms are interesting, at present it is unclear if they can be applied to sequences in the same way as the underlying RL algorithms used in typical RLHF pipelines.

In this work we rectify this difference by deriving DPO within the token-level MDP setting present in large language models using the usual form of binary preference-feedback. We then show that DPO training implicitly learns a token-level reward function, for which the language models logits define the optimal Q function, or expected total future reward. We then demonstrate that DPO is able to flexibly model any possible dense reward function within the token MDP. 

Empirically, we use our theoretical derivations to justify three practical insights which we believe to be of use to the community. First, we show that despite being derived as a contextual bandit, the implicit rewards of a DPO model have a per-token interpretation. Second, we demonstrate that likelihood search over a DPO model is analogous to searching over a reward function during decoding as done by contemporary works \citep{liu2023making, feng2024alphazerolike}. Finally, we identify the choice of initial policy and reference distribution as being important in determining the trajectory of implicit rewards during training. 



\section{Related Work}

The problem of aligning policies with human intent using preference feedback has been a long studied problem in reinforcement learning \citep{akrour2011preference, wilson2012bayesian}. While the primary focus of RLHF was originally in control \citep{christiano2017deep}, following the success of \citet{ziegler2020finetuning} it has recently been broadly adopted by the language modeling \citep{ouyang2022training, nakano2021webgpt,stiennon2022learning, bai2022training} and even vision communities \citep{black2023training, lee2023aligning}. Most works in RLHF optimize a learned reward function, used only at the end of generation, with a policy graident style-method. Such approaches have been known to be unstable \citep{engstrom2020implementation} and hard to scale, while at the same time theoretically existing at an unusual intersection between contextual bandits and RL. In response, several direct alignment methods \citep{rafailov2023direct, azar2023general, zhao2023slic} have been developed which simplify the RLHF pipeline by learning a policy from preference data without an intermediate reward function. Such methods however, derived solely as contextual bandits, leave several theoretical and practical questions unanswered which we seek to address.

First, though direct alignment methods treat the LLM as a bandit, prior works have demonstrated that it is possible to use dense rewards \citet{zelikman2022star, chan2024dense, pan2023let} or even approximate dynamic programming \citep{snell2022offline}. Moreover, using the regret-model of preferences \citep{knox2023models, knox2024learning}, Contrastive Preference Learning \citep{hejna2024contrastive} is able to use direct alignment for general MDPs, instead of the specific token MDP used in RLHF. Our work shows how DPO can be interpreted as optimizing a per-token reward function, which in practice is restricted to the family of optimal advantage functions.

Second, if DPO does not learn a reward function, can we still use its reward or value? Prior works have considered using best-of-K \citep{mudgal2023controlled} or tree search \citep{liu2023making} for alignment with a value function \cite{kim2022critic, li2017learning} or discriminator \citep{yang2021fudge}. Using the implicit reward, we show that likelihood search results in a similar solution for direct alignment. 

Our work builds on foundational knowledge in maximum entropy RL \citep{ziebart2010modeling} and inverse RL \citep{ziebart2008maximum, ng1999policy, cao2021identifiability}. In particular, we leverage the mapping between $Q$-functions and reward functions under a fixed policy as first done in inverse RL by \citet{garg2022iqlearn}. Related to our work, \cite{nachum2017bridginggapvaluepolicy} uses similar derivations for reinforcement learning in control and \citet{watson2023coherent} does so for inverse RL from demonstration. \citet{hejna2024inverse} exploit this relationship for RLHF. While related, these works still require an additional loop of reinforcement learning optimization, which we dispose of in our formulation of feedback learning for LLMs.  In the LLM domain, \cite{yu2024mathcalbcodervaluebaseddeepreinforcement} uses pre-trained models as priors for Q-learning, while \citet{cundy2023sequencematch} considers a similar formulation for imitation learning. In this work instead, we formulate preference-based learning as Q-learning.

\section{Preliminaries}

In this section we first define the per-token MDP for large language models, and then describe how it relates to classic RLHF approaches and direct alignment algorithms, specifically DPO. We operate in the typical RLHF setting where we have a dataset $\mathcal{D}=\{(\textbf{x}^{(i)}, \textbf{y}^{(i)})\}_{i=1}^N$ of language prompts $\textbf{x}$ and target answers $\textbf{y}$, which can each individually be broken down into a sequence of tokens, for example $\bx=(x_0, \ldots, x_m)$, from a fixed discrete vocabulary $\mathcal{A}$. Throughout this section we will use the $\bx$, $\by$ notation for the contextual bandit framing where the entire response $\by$ is the action, but will use state $\bs$ and action $\ba$ notation from RL literature for describing sequences at the token-level.

\subsection{The Token-level MDP for Large Language Models}\label{section:tokenMDP}
We define the token level MDP as a tuple $\mathcal{M} = (\mathcal{S}, \mathcal{A}, f, r, \rho_0)$, where the state space $\mathcal{S}$ consists of all tokens generated so far (i.e. $\bs_t=\{x_0, \ldots, x_m, y_0, \ldots, y_t\}$) and the action space is the vocabulary of tokens $\mathcal{A}$. The dynamics $f$ are the deterministic transition model between tokens $f(\bs,\ba) = \bs|\ba$, where $|$ is concatenation. The initial state distribution $\rho_0$ is a distribution over prompts $\bx$, where an initial state $\bs_0$ is comprised of the tokens from $\bx$. In RLHF, the reward function is learned from human feedback over preferences between responses which we will denote using trajectories $\tau$ at the token level. As is typically done \citep{ziegler2020finetuning, stiennon2022learning}, we assume that preference trajectories start at the same state (initial propmpt) and end in a terminal state (\textbf{EOS} token), from which future rewards are zero. 
In this token level MDP, the corresponding Bradley-Terry preference model \cite{bradley1952rankanalysis, christiano2017deep} is 
\begin{equation}\label{eq:dense-bradley-terry}
    p^*(\tau^w \succeq \tau^l)=\frac{\exp\left(\sum_{i=1}^N r(\bs_i^w, \ba_i^w)\right)}{\exp\left(\sum_{i=1}^N r(\bs_i^w, \ba_i^w)\right)+ \exp\left(\sum_{i=1}^M r(\bs_i^l, \ba_i^l)\right)}.
\end{equation}
which gives the probability that the ``win'' trajectory $\tau^w$ of length $N$ is preferred to the ``loss'' trajectory $\tau^l$ of length $M$. Now that we have defined the token level MDP, we can show how it relates to both classic and direct alignment RLHF methods.

\subsection{The Classical RLHF Methods}

Most classical RLHF approaches \citep{ziegler2020finetuning, bai2022constitutional, ouyang2022training} first learn a reward function from human feedback on prompt and response pairs $(\bx, \by^w, \by^l)$, then optimize it with a policy gradient-based method like PPO \citep{schulman2017proximal} with an entropy-bonus using the following KL-constrained RL objective
\begin{equation}\label{eq:multi_step_RL}
\max_{\pi_{\theta}}  \mathbb{E}_{a_t \sim \pi_{\theta}(\cdot | \bs_t)}\left[\sum_{t=0}^T
(r(\bs_t, \ba_t) + \underbrace{\beta\log\piref(\ba_t
|\bs_t)}_{\text{KL penalty}}) + \beta\mathcal{H}(\pi_{\theta})|\bs_0\sim\rho(\bs_0)\right]
\end{equation}

where $\piref$ is a reference policy, often resulting from supervised finetuning, from which the learned policy should not significantly deviate. However, in classic RLHF methods the reward function is learned as a contextual bandit with the preference model
\begin{equation*}
\label{eq:bandit_pref}
    p^*(\by^w \succeq \by^l) = \frac{\exp r(\bx,\by^w)} {\exp r(\bx,\by^w) + \exp r(\bx,\by^l)}
\end{equation*}
and is thus only applied at the final timestep for the last action where $\ba$ is \textbf{EOS}. In practice the actual reward used in the token-level PPO is 
\begin{equation}\label{eq:token_reward}
    r(\bs_t, \ba_t) =
    \begin{cases}
      \beta \log \piref(\ba_t | \bs_t), & \text{if $\bs_{t+1}$ is not terminal} \\
      r(\bx, \by) +\beta \log \piref(\ba_t | \bs_t), & \text{if $\bs_{t+1}=\by$ is terminal} 
    \end{cases}
\end{equation}
in a maximum entropy formulation. This leads to an interesting contradiction where the reward function $r$ is treated like a bandit, but the actual RL value function and optimization is done per-token in practice.

\subsection{Direct Preference Optimization}
Unlike classical RLHF, DPO, as derived in \citet{rafailov2023direct}, stays entirely within the contextual bandits setting entirely and also uses the bandit-based preference model in \cref{eq:bandit_pref}. To circumvent the need for an RL algorithm, DPO uses the well-known closed form solution to the KL-contextual bandit version of the RL problem posed in \cref{eq:multi_step_RL} \citep{ziebart2008maximum, levine2018reinforcement}:
\begin{equation*}
\pi^*(\by|\bx) = \frac{1}{Z(\bx)}\piref(\by|\bx)e^{r(\bx, \by)}
\end{equation*}
where $\pi^*$ is the optimal policy and $Z(\bx)$ is the partition function that normalizes it. DPO re-arranges this equation to solve for reward as $r(\bx,\by) = \beta \log \pi^*(\by|\bx) - \beta \log \piref(\by|\bx) - Z(\bx)$. Substituting this relationship into the standard binary cross-entropy loss function used for reward modeling yields the DPO loss equation as the partition function $Z(\bx)$ cancels from the Bradley Terry model.
\begin{equation}\label{eq:optimum_model}
    \mathcal{L}_\text{DPO}(\pi_{\theta}; \piref) = -\mathbb{E}_{(\bx, \by^w, \by^l)\sim \mathcal{D}}\left[\log \sigma \left(\beta \log \frac{\pi_{\theta}(\by^w\mid \bx)}{\piref(\by^w\mid \bx)} - \beta \log \frac{\pi_{\theta}(\by^l\mid \bx)}{\piref(\by^l\mid \bx)}\right)\right]
\end{equation}
For brevity we use $\sigma$ to denote the logistic function. In the next section, we show how an alternative derivation of DPO can also cast its optimization within the token-level MDP.

\section{Theoretical Insights}

In this section we explore how DPO can theoretically be cast into the token-level MDP, and explore the consequences of doing so. First, we provide a token level derivation of DPO under the assumptions in \cref{section:tokenMDP}. Next, we show that even in the token MDP, DPO is able to fit any reward function in the multi-step Bradley Terry preference model \cref{eq:dense-bradley-terry}. Ultimately, this shows that DPO can potentially be used for more sequential optimization tasks, like multi-turn interactions or even multi-modal generation.

\subsection{DPO as a $Q$-function in the Token Level MDP}

\noindent \textbf{RL in the Token-level MDP.} While the original derivation of DPO relies on the fact that $Q^*(\bx, \by) = r(\bx, \by)$, this relationship does not hold in the token-level MDP. To resolve this, we need to develop new mathematical results that will allow us to relate the reward function in the Token-level Bradley Terry model \cref{eq:dense-bradley-terry} to the corresponding optimal policy $\pi*$. In the general maximum entropy RL setting, the fixed point solution of \cref{eq:multi_step_RL} is given by \citep{ziebart2010modeling} as
\begin{equation}\label{eq:policy}
    \pi^*(\ba_t|\bs_t) = e^{(Q^*(\bs_t, \ba_t)-V^*(\bs_t))/\beta}
\end{equation}
where $\pi^*(\ba|\bs)$ is the optimal policy and $Q^*(\bs, \ba)$ is the optimal Q-function which models the total future reward from $(\bs, \ba)$ under $\pi^*$. The optimal value function $V^*$ is a function of $Q^*$, 
\begin{equation}\label{eq:value}
    V^*(\bs_t) =  \beta\log \sum_{\ba \in \mathcal{A}} e^{Q^*(\bs_t, \ba)/\beta} 
\end{equation}
such that the policy $\pi^*$ integrates to one. Unfortunately unlike in the bandits setting this relationship gives us no specific information about the reward function $r$ at a single state action pair since the optimal policy optimizes for total future returns as estimated by $Q$. To do so, we will need to consider the relationship between $Q^*$ and $r$. 

\noindent \textbf{From $r$ to $Q^*$.}
The relationship between future returns and the current timestep is captured by the belmman equaitons which are satisifed by any valid Q-function. We write this below for the optimal policy $\pi^*$ under the reward $r$ with a KL divergence penalty:
\begin{equation}\label{eq:critic}
    Q^*(\bs_t, \ba_t) =
    \begin{cases}
      r(\bs_t, \ba_t) + \beta \log \piref(\ba_t| \bs_t) + V^*(\bs_{t+1})  , & \text{if $\bs_{t+1}$ is not terminal} \\
      r(\bs_t, \ba_t) + \beta \log \piref(\ba_t| \bs_t), & \text{if $\bs_{t+1}$ is terminal} 
    \end{cases}
\end{equation}
We can then rearrange the bellman equation for the optimal $Q$-function in terms of the reward. This style of relationship was first explored by \citet{garg2022iqlearn} in imitation learning and later in \citet{hejna2024inverse} for preference-based RL. However, these works \emph{require} the use of a discount factor $\gamma < 1$ which is typically not used in RLHF. In the appendix we prove the following  Lemma which shows that this relationship is indeed one-to-one in the token MDP as well.

\begin{lemma} \label{lemma:r_to_q} Under mild assumptions, there is a bijection between reward functions $r(\bs_t, \ba_t)$ and corresponding optimal Q-functions $Q^*(\bs_t, \ba_t)$ in the token MDP.
\end{lemma}

This leads us to a rather interesting conclusion -- that an LLM is \emph{always} the optimal soft Q-functions for \emph{some} reward function in the token MDP. Consider any LLM which outputs logits $l_\theta$ and temperature parameter $\beta$. As is common practice, we take the sampling policy $\pi$ to be the softmax over tokens modulated by temperature parameter $\beta$ -- which is precisely \cref{eq:policy} where $Q^* = l_\theta$ because the value optimal function $V^*$ is precisely $\beta \log Z(\bs_t)$, normalizing the distribution. The corresponding reward function may not be smooth or well-behaved. Notably, the logits have a free parameter due to the softmax. While this free-parameter results in the same optimal policy per later arguments, it means the sequence of values may not be smooth. The question then becomes how to finetune the LLM such that it is the optimal Q-function for a reward function $r$ that aligns with human preferences. To do so, we will complete our derivation of DPO in the token MDP.

\textbf{DPO learns our best estimate of $Q^*$.} Now that we have established a bijection between $r$ and $Q^*$, we can derive a token-level version of DPO to align the implicit reward, induced by the $Q$ function represented by the language model, with that of the best estimate of reward, according to Bradley-Terry model in \cref{eq:dense-bradley-terry}. To do so, we need to represent the sum of rewards first in terms of the $Q$-function $Q^*$, and then in terms of the policy $\pi^*$. We complete the first step by inverting the Bellman equation in \cref{eq:critic} and substituting it into the sum of rewards over a trajectory $\tau = \{\bs_1, \ba_1, \ldots, \ba_{T-1}, \bs_T\}$.
\begin{align*}
    \sum_{t=0}^{T-1}r(\bs_t, \ba_t) 
    &= \sum_{t=0}^{T-1}\left(Q^*(\bs_t, \ba_t) - \beta \log \piref(\ba_t | \bs_t) - V^*(\bs_{t+1})\right) = \\ 
    &= Q^*(\bs_0, \ba_0) - \beta \log \piref(\ba_0 | \bs_0) + \sum_{t=1}^{T-1}Q^*(\bs_t, \ba_t) - V^*(\bs_{t})  - \beta \log \piref(\ba_t | \bs_t)
\end{align*}
The equality follows from $V^*(\bs_T)=0$ and re-arranging the sum to isolate $t=0$. As $V^*$ is written entirely in terms of $Q^*$ and $\beta$ per \cref{eq:value}, we have expressed the sum of return over the sequence just in terms of $Q^*$. Next, we exchange $Q^*$ for $\pi^*$. We can log-linearize \cref{eq:policy} as $\beta \log \pi^*(\ba_t| \bs_t) = Q^*(\bs_t, \ba_t) - V^*(\bs_t)$. This is equivalent to stating that the language model probabilities are just the softmax over $l_\theta = Q^*$ with temperature $\beta$. Continuing from the above, with this substitution we get
\begin{equation*}
    = Q^*(\bs_0, \ba_0) - \beta \log \piref(\ba_0 | \bs_0) + \sum_{t=1}^{T-1} \beta \log \frac{\pi^*(\ba_t|\bs_t)}{\piref(\ba_t|\bs_t)} =V^*(\bs_0) + \sum_{t=0}^{T-1} \beta \log \frac{\pi^*(\ba_t|\bs_t)}{\piref(\ba_t|\bs_t)} 
\end{equation*}
where the final step results from adding and subtracting $V^*(\bs_0)$ and applying the substitution again. Now, this representation for the sum of rewards in terms of the optimal policy can be directly substituted into the preference model in \cref{eq:dense-bradley-terry}, where the $V^*(\bs_0)$ term will cancel just as $Z(\bx)$ did in the original DPO derivation assuming $\tau^w$ and $\tau^l$ start at the same state $\bs_0$, giving us the policy-induced preference model \looseness=-1
\begin{equation}
\label{eq:policy_pref}
    p_{\pi^*}(\tau^w \succeq \tau^l) = \sigma \left(\sum_{t=0}^{N-1} \beta \log \frac{\pi^*(\ba_t^w|\bs_t^w)}{\piref(\ba_t^w|\bs_t^w)} - \sum_{t=0}^{M-1} \beta \log \frac{\pi^*(\ba_t^l|\bs_t^l)}{\piref(\ba_t^l|\bs_t^l)}\right).
\end{equation}
To derive the final DPO loss function, we can take the KL-divergence between the empirical preference model of our dataset $p_\mathcal{D}$ and the preference model implied by a learned policy $p_{\pi_\theta}$, $\kl (p_\mathcal{D} || p_{\pi_\theta})$. This results in 
\begin{equation}\label{eq:DPO}
    \mathcal{L}(\pi_{\theta}, \mathcal{D}) = -\mathbb{E}_{(\tau_w, \tau_l)\sim \mathcal{D}}\left[\log \sigma\left(\left( \sum_{t=0}^{N-1}\beta \log\frac{\pi^*(\ba_t^w| \bs_t^w)}{\piref(\ba_t^w| \bs_t^w)}\right)- \left( \sum_{t=0}^{M-1}\beta \log\frac{\pi^*(\ba_t^l| \bs_t^l)}{\piref(\ba_t^l| \bs_t^l)}\right)\right)\right]
\end{equation}
In the next section we demonstrate that DPO can learn any dense reward function in the token-level MDP.

\subsection{Token-Level DPO Can Parameterize Any Dense Reward Function.}\label{secttion:dpoisuniversal}

In the previous section we derived DPO using the bijection between reward functions and optimal $Q$-functions uniquely available in the token-level MDP. An alternative view of DPO casts it as restricting the learned reward function such that it belongs to the class optimal advantage functions $A^*(\bs,\ba) = Q^*(\bs,\ba) - V^*(\bs)$ from which an optimal policy is readily obtained per \cref{eq:policy}. Here we show that this restriction does not limit the class of reward functions we can represent. We begin by expanding the definition of equivalency used in \citet{rafailov2023direct} to the broader class of potential-based reward shaping functions:

\begin{definition}\label{def:equivalence}
Two reward functions $r(\bs_t, \ba_t)$ and $r'(\bs_t, \ba_t)$ are equivalent if there exists a potential function $\Phi(\bs)$, such that $r'(\bs_t, \ba_t) =r(\bs_t, \ba_t) + \Phi(\bs_{t+1})  - \Phi(\bs_{t})$.
\end{definition}

In \citet{ng1999policy}'s seminal work, the authors proved that two equivalent reward functions defined per \cref{def:equivalence} have the same optimal policy. By log-linearizing the optimal policy fixed point in \cref{eq:policy} and substituting in the Bellman equation from \cref{eq:critic} \citep{nachum2017bridginggapvaluepolicy, watson2023coherent}, we have
\begin{equation}
\label{eq:advantage}
    \beta \log\frac{\pi^*(\ba_t | \bs_t)}{\piref(\ba_t|\bs_t)} = r(\bs_t, \ba_t) + V^*(\bs_{t+1}) - V^*(\bs_{t}).
\end{equation}
This is precisely the optimal advantage function, where $V^*$ directly follows the form of a potential shaping function. \citet{watson2023coherent} first used this derivation to arrive at a ``coherent'' reward function and follow-ups arrived at the same conclusion by noting that using the advantage as reward preserves the optimal policy \citep{knox2024learning, hejna2024contrastive}. Unlike prior works, however, we demonstrate that this re-parameterization also leads to the same exact \emph{preference} distribution as $r$. 

\begin{theorem}\label{theorem:equiv} Given a reference policy $\piref$ and a parameter $\beta>0$ all reward classes consistent with the Plackett-Luce (and Bradley-Terry) models in \cref{eq:dense-bradley-terry} can be represented with the a re-parameterization of the form 
\begin{equation}\label{eq:reward_param}
    r(\bs, \ba) = \beta \log \pi(\ba| \bs) - \beta \log \piref(\ba| \bs)
\end{equation}
within the token MDP where $V^*(\bs_t) = 0$ for all terminal states.
\end{theorem}
\begin{proof}







Above we derived the invariance of the optimal policy under the re-parameterization. The preference model can be shown to be invariant by substituting and following the same steps used to arrive at \cref{eq:policy_pref} in the last section, or by following Definition 1 from \citet{watson2023coherent}. \looseness=-1
\end{proof}
Interestingly, in practice, the potential function $\Phi(\bs_t)$ represents the free parameter in the logits of the language model. An equal shift along all logits yields the same policy, but different Q-functions and corresponding rewards. The above Theorem proves that all of these are in the same equivalence class and induce the same set of preferences. 

Moreover, this Theorem implies that we can use DPO to learn the optimal policy for any per-token reward function, provided preference queries start at the same state and end at a terminal state. In addition, DPO \emph{always} fits an optimal advantage function for \emph{some} reward which is responsible for credit assignment. Thus, the training data determines how close the learned advantage corresponds to that of the true reward. This is in contrast to methods that estimate the reward function and then additionally employ some policy improvement mechanism. Which algorithm performs better remains largely an open or empirical question. 

The above derivations cast a language model as a Q function in the discrete token-level MDP. While this interpretation does not generally hold in continuous spaces, we can extend many of our results to other specially structured MDPs, like those present in diffusion. See Appendix \ref{appendix:diffusion} for more thorough treatment.





    
\section{Practical Insights}

In this section we discuss the empirical implications of our theoretical analysis. First, we qualitatively show that DPO can learn per-token credit assignment. Next, we use the derivations of the prior section to connect guided decoding and search-based algorithms, such as MCTS, to likelihood-based search on the DPO policy and empirically validate these results. Finally, (for the first time), we mathematically explain the phenomenon of decreasing likelihoods during DPO training, observed in the research and industry community.

For all empirical evaluations we use the Pythia 2.8B model \cite{biderman2023pythia} and the Reddit TL;DR summarization dataset \cite{stiennon2022learning}. We use the default hyper-parameters from the original public DPO implementation, unless otherwise stated. 

\subsection{Does DPO Learn Credit Assignment?}
\begin{figure}
    \centering
    \includegraphics[width=0.495\textwidth]{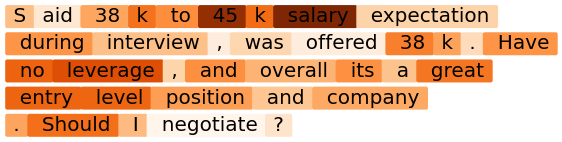}
    \includegraphics[width=0.495\textwidth]{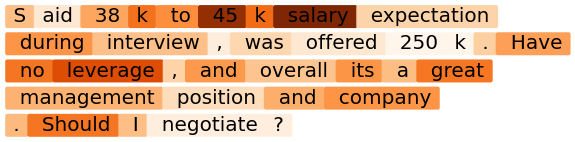}
    \caption{Credit assignment in DPO based on answer-level feedback. We provide two summaries to a Reddit post about a job interview. The left is the base response and on the right we have introduced errors in the salary range and the position level. Each token is colored corresponding to the DPO implicit reward as expressed in Eq. \ref{eq:reward_param} (darker is higher), using the trained model. We see that the model correctly highlights the erroneous statements, without much change to the value of the other tokens, which indicates the ability to do credit assignment.}
    \label{fig:credit assignment}
\end{figure}

In the previous section we outlined how the trained DPO policy represents an optimal Q-function for some reward that optimizes the preference equation. In this section, we evaluate qualitatively if the DPO-trained model is able to learn credit assignment from trajectory feedback. We begin with a generic set of Reddit posts for the TL;DR test dataset, which we provide in Appendix \ref{appendix:reddit_post} with additional examples. In our representative example the user discusses an employment negotiations situation. Two answers are shown in Figure \ref{fig:credit assignment}. The base summary, which is correct is provided on the left. On the right we modify the summary by introducing a higher-level position and a corresponding higher salary. For each token in both answers we compute the DPO reward (equivalently the advantage function or ``coherent'' reward \citep{watson2023coherent}),  $r(\mathbf{s}, \mathbf{a}) = \beta \log \pi_{\theta}(\mathbf{s}| \mathbf{a}) - \beta \log \piref(\mathbf{s}| \mathbf{a})$, where $\pi_{\theta}$ as outlined in Theorem  \ref{theorem:equiv} (here $\pi_{\theta}$ is our DPO-trained model and $\piref$ is the SFT model). In Figure \ref{fig:credit assignment} each token is colored proportionally to this reward. We see that the model successfully identifies the tokens corresponding to the erroneous statements, while still maintaining comparable values for the rest, which is indicates that it can do credit assignment. Moreover, we see that within the context of the first error ("250K" salary) the model still allocates reasonable values to the rest of the tokens and specifically identifies the second error "management position". This is a promising sign of the ability to do "stitching" \cite{levine2020offline} i.e. a form of combinatorial generalization from offline data.  If this is the case, our findings could be significant for the use of reinforcement learning and RLHF in LLMs, particularly for compositional tasks, such as code and reasoning. At the same time, in the recently introduced RewardBench \cite{lambert2024rewardbench}, DPO models have demonstrated strong performance as classifiers on reasoning tasks. We believe these are encouraging results, which warrant further large-scale study beyond our qualitative observations.

\subsection{Connecting Guided Decoding and Search to Likelihood-Based DPO Optimization}
Recently Large Language Models have been combined with search algorithms during the inference stage \cite{mudgal2024controlled, feng2024alphazerolike, huang2024deal,  liu2023dont}, which have found to improve the quality of responses over standard next token decoding. Following the standard literature, these methods rely on a (usually sparse) reward signal or model $r_{\theta}(\mathbf{s_t}, \mathbf{a}_t)$ which they use to train a separate value function $V_{\theta}(\mathbf{s}_t)$. During inference time they deploy a graph-search algorithm in the token MDP as outlined in Section \ref{section:tokenMDP} to maximize the sum of rewards. Let us consider the search problem outlined in Eq. \ref{eq:multi_step_RL} with a partial expansion of length $K$: 
\begin{equation}\label{eq:likelihood_search}
    \max_{\ba_0, \ldots, \ba_K}r(\mathbf{s}_0, \mathbf{a}_0) + \beta\log\piref(\mathbf{s}_0, \mathbf{a}_0) + \ldots + r(\mathbf{s}_t, \mathbf{a}_t) + \beta\log\piref(\mathbf{s}_K, \mathbf{a}_K) + V^*(\mathbf{s}_{K+1})
\end{equation}
where $V^*$ is the optimal corresponding value function. Now, if we directly substitute the reward representation from Eq. \ref{eq:advantage} into the above and considering a telescoping sum, with some standard algebra, we obtain that the above objective is equivalent to
\begin{equation}
   \max_{\ba_0, \ldots, \ba_K} -V^*(\mathbf{s}_0) + \beta \log\pi^*(\mathbf{a}_0|\mathbf{s}_0) + \ldots + \beta \log\pi^*(\mathbf{a}_K|\mathbf{s}_K) 
\end{equation}
where $\pi^*$ is the corresponding optimal policy. Now, since the starting state is fixed (it's given by the prompt) we have that a search algorithm based on the conservative reward function of the RLHF objective and the corresponding optimal value policy is equivalent to likelihood search on the corresponding optimal policy. 
We empirically verify this property in Fig. \ref{fig:likelihood_plot}, which shows the win rate of DPO models trained with three different $\beta$ values against the preferred summary in the test dataset. We see that a 5-beam search improves win-rates by 10-15\% over the base policy (1-beam), which is comparable to the value-function guided search improvements reported in \cite{mudgal2024controlled}. Interestingly, we see performance degrade with higher number of beams. Increasing the number of beams also produces answer with exploding length, which is a sign of reward over-optimization \cite{gao2022scaling, park2024disentangling, rafailov2024scalinglawsrewardmodel} and would explain the degradation in performance. These observations are consistent with out formulation of beam search as a search over a learned reward function.


These findings are consistent with the result of the recently proposed V-STaR algorithm \cite{hosseini2024vstar}, which combines the approach of STaR \cite{zelikman2022star} with a DPO trained verifer. At inference time, the STaR model produces several candidate reasoning chains (plans) which are ranked by the DPO verifier likelihood. This can be seen as a form of likelihood based search as in Eq. \ref{eq:likelihood_search}, however instead of directly searching on the DPO model, it uses the STaR model as a proposal distribution. We hypothesize this is beneficial in preventing reward hacking, which is potentially an issue with deeper search as shown in Fig. \ref{fig:likelihood_plot}.

\begin{figure}
    \centering
    \includegraphics[width=0.475\textwidth]{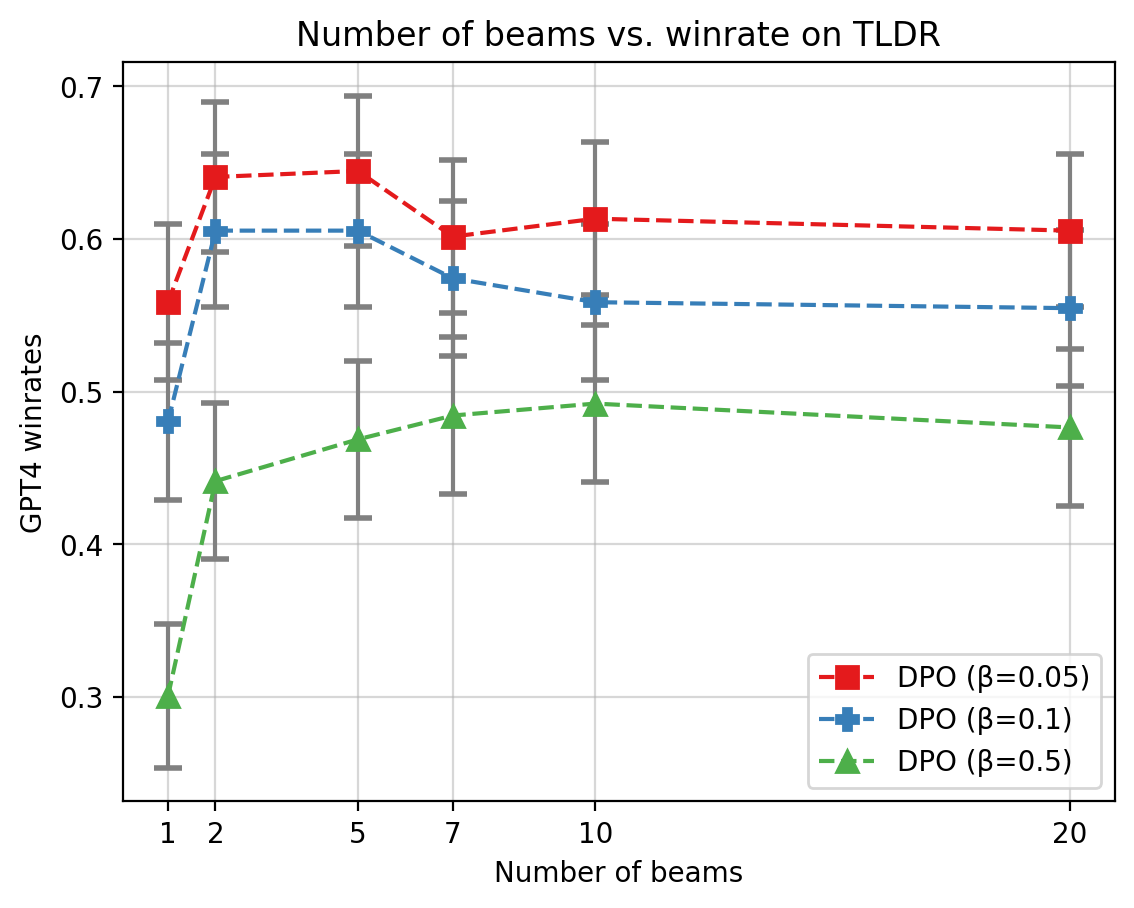}
    \includegraphics[width=0.475\textwidth]{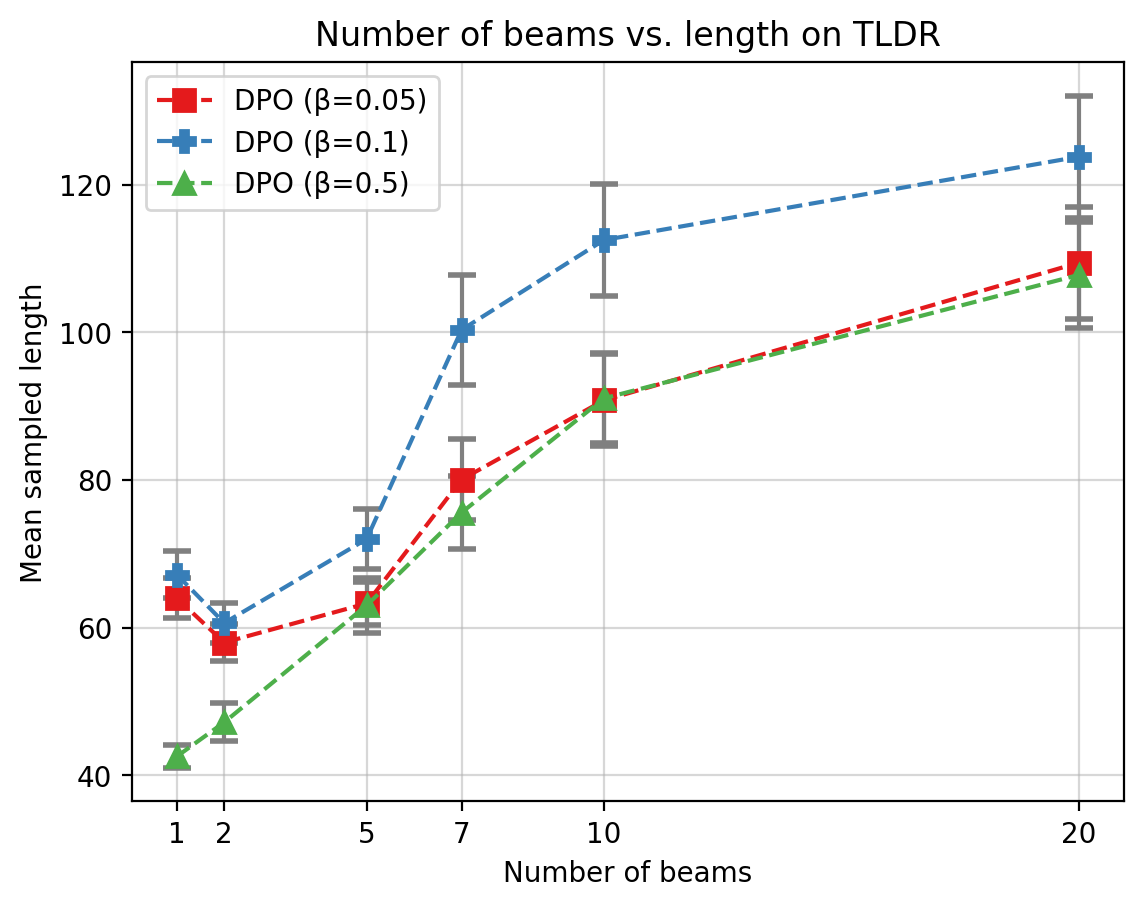}
    \caption{Model performance using beam search. \textbf{Left}: Win rate of the model generated summaries over the preferred summary on 256 held-out test prompts from the Reddit TL;DR dataset, as evaluated by GPT 4. \textbf{Right}: The average answer length based on number of beams. We see exploding verbosity with more than 5 beams, which also leads to lower model win rates, despite GPT4's well-know preference length bias.}
    \label{fig:likelihood_plot}
\end{figure}

\subsection{Connections Between Proxy Tuning and Reinforcement Learning}
Several recent works \cite{mitchell2023emulatorfinetuninglargelanguage, liu2024tuninglanguagemodelsproxy, liu2024decodingtimerealignmentlanguagemodels} have proposed an approach of inference-time model alignment through a proxy guidance model. These approaches start with a (unaligned) base model $\pi_{\text{base}}$ and a proxy model $\pi_{\text{proxy}}$ and a target distribution reference model $\pi_{\text{ref}}$. The inference time re-alignment of the base model is carried by re-weighting the conditional probabilities of each token:
\begin{equation}
    \pi(\mathbf{a}|\mathbf{s}_t)\propto \pi_{\text{base}}(\mathbf{a}|\mathbf{s}_t)\left(\frac{\pi_{\text{proxy}}(\mathbf{a}|\mathbf{s}_t)}{\pi_{\text{ref}}(\mathbf{a}|\mathbf{s}_t)}\right)^{\beta}
\end{equation}

Under our considerations from the prior chapter, then this becomes equivalent to 
\begin{equation}
    \pi(\mathbf{a}|\mathbf{s}_t)\propto \pi_{\text{base}}(\mathbf{a}|\mathbf{s}_t)\exp(\beta(Q^*(\mathbf{s}_t, \mathbf{a})-V^*(\mathbf{s}_t)))
\end{equation}

where $\beta(Q^*(\mathbf{s}_t, \mathbf{a})-V^*(\mathbf{s}_t)$ is the optimal implicit advantage from the proxy tuning model. That is our theoretical results allows us to tie the realignment approaches of \cite{mitchell2023emulatorfinetuninglargelanguage, liu2024tuninglanguagemodelsproxy, liu2024decodingtimerealignmentlanguagemodels} to recent works which explicitly train critic models \cite{mudgal2024controlled} for token-level decoding.

\subsection{Likelihoods should decrease when using DPO.}
A surface level interpretation of DPO would lead one to believe it increases the likelihood of chosen responses, while decreasing the likelihood of rejected responses. This however, does not account for a well observed phenomena in which the likelihood of the chosen responses actually \textit{decrease} over time \citep{pal2024smaug}. This is illustrated on the left half of \cref{fig:phenomena}, which we show that when performing SFT before DPO, the implicit rewards of both the chosen and rejected response decline, though the margin between them increases. However, given a MaxEnt RL framing, this phenomena may be expected. 

Consider the expected log ratio (or implicit reward) of a policy under the reference model, which is often measured during training. Algebraic manipulation yields the following relationship:
\begin{equation}
\label{eq:logratio}
    \E_{\ba\sim \piref(\cdot|\bs)}\left[\beta \log \frac{\pi(\ba|\bs)}{\piref(\ba|\bs)}\right] = - \beta \kl \left(\piref(\cdot | \bs) || \pi(\cdot | \bs) \right)
\end{equation}
At the beginning of training when $\pi = \piref$, the implicit rewards are trivially zero. However at the end of training, assuming $\piref \ne \pi^*$, the KL-divergence is necessarily positive, indicating that the implicit rewards must decrease in expectation to converge. This means that the average implicit rewards \textit{should} go down when starting from the SFT model.  In fact, on the left side of \cref{fig:phenomena} we show that when one \emph{does not} SFT before DPO, there is little discernible trend in the average implicit reward and the implicit rewards of the chosen responses remain above zero. In fact, this trend also holds for CPL \citet{hejna2024contrastive} for the general MDP, where the implicit rewards actually increase if SFT is not used.

One might realize that the previous analysis does not necessitate that the implicit rewards of the chosen must decrease, just that the implicit rewards must decrease on average. However, in practice it is common place (and recommended by \citet{rafailov2023direct}) to SFT on only the chosen responses to form $\piref$. For this section only we will call this choice of reference $\piref^w$. Substituting $\piref^w$ into \cref{eq:logratio}, we can see that when SFTing on the positive answers the implicit rewards of the chosen responses \emph{must} go down because at convergence as  $\E_{\piref^w }[\beta \log \pi^* - \beta \log \piref^w] = - \beta\kl(\piref^w || \pi^*)$.  

\textbf{Based on this derivation and choice of $\piref^w$, the likelihood of the chosen response should decrease in the process of DPO training.} 


While choosing $\piref = \piref^w$ is done in practice \citep{rafailov2023direct}, it does mean that DPO will decrease the likelihood of all data in favor of extrapolated responses, which could cause over-fitting. Moreover, now that we have provided a derivation of DPO in the token-level MDP, one might expect it to exhibit characteristics like an RL algorithm -- namely that the implied $Q$-function monotonically increases over time. However, this is not necesarily the case. Note that per analysis in Section \cref{section:tokenMDP}, DPO can be viewed as adjusting the reward (or advantage) from which the optimal policy is deterministically mapped within the token-level MDP. DPO does not train a policy to maximize reward, and thus we do not argue about whether its implied value functions should increase or decrease over time.

\begin{figure}
    \centering
    \includegraphics[width=0.45\textwidth]{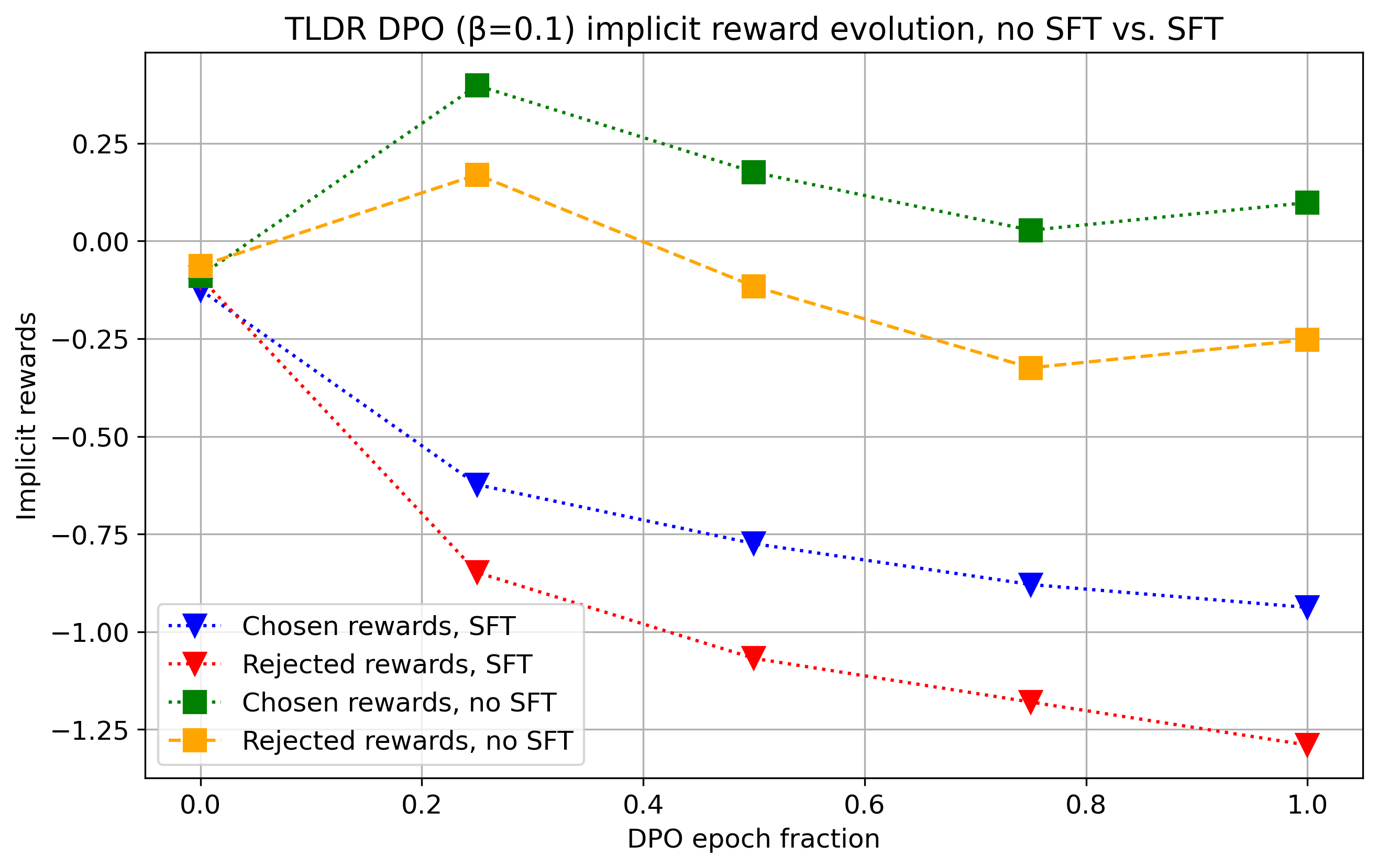}
    \includegraphics[width=0.45\textwidth]{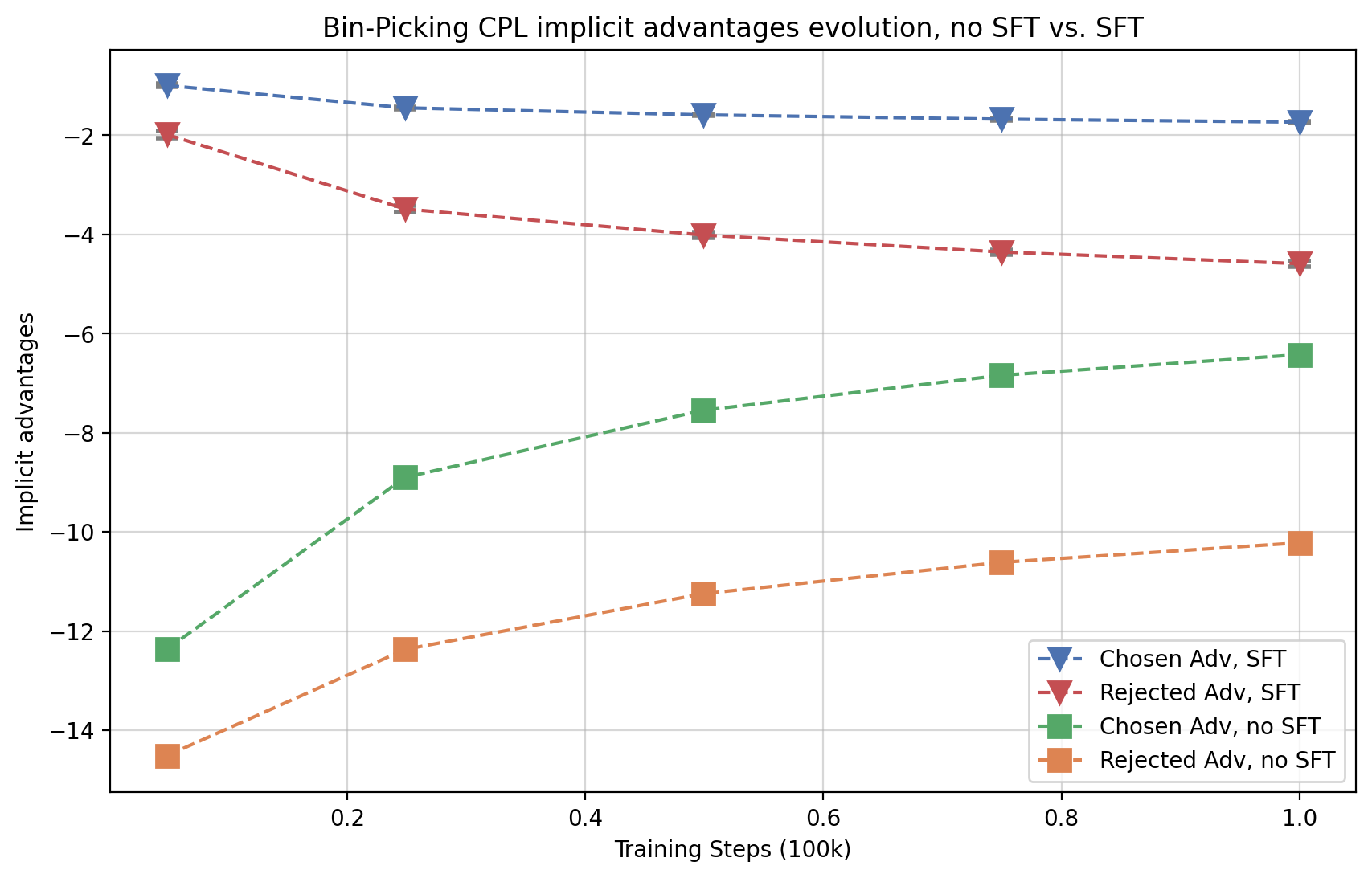}
    \caption{The evolution of implicit rewards for DPO on TLDR (left) and CPL on the bin-picking dataset (right) during training. We see that when we start with SFT, reward values decrease, whereas starting without SFT causes implicit rewards to be positive for DPO and increase for CPL. \looseness=-1}
    \label{fig:phenomena}
\end{figure}


\section{Discussion}
In this work we formulated the DPO optimization algorithm as learning an optimal Q-function, which is represented by an LLM. This formulation and our results provide theoretical justification for empirically observed DPO training phenomena, which are not explained by the original bandit formulation. We further link and unify a family of proposed new LLM search algorithms by likelihood search under DPO and show comparable empirical gains by a simple 1-line code change to using beam search. Most importantly, we show qualitative early signs that DPO is able to learn credit assignment directly from feedback data. While larger-scale empirical exploration is necessary, we believe this an encouraging early sign. Our results indicate a number of promising future directions to explore:

\noindent\textbf{Learning intermediate reasoning from outcome feedback:} Recent works have shown promising results on that front \cite{pang2024iterativereasoningpreferenceoptimization, hwang2024selfexploreavoidpitimproving}.

\noindent\textbf{Multi-turn conversations:} Teaching language models to be an interactive conversationalists has been difficult, as RLHF is optimized as a single-turn bandit formulation. Moreover, classical methods, such as PPO, are not applicable in this setting. Recent work by \citet{andukuri2024stargate} has shown success in this domain using STaR and extending DPO to multi-turn conversational trees is a promising direction. 

\noindent\textbf{Agentic LLMs:} LLM agents, such as WebGPT \citep{nakano2022webgpt} are equipped to take autonomous actions, such as browsing the Web and collecting information before providing an answer. The user then provides feedback based on the final output. Our derivations indicate that DPO training (on the full model trajectories) could learn optimal exploration behaviour. Recent works \cite{song2024trial, xi2024agentgymevolvinglargelanguage} shows promise in that direction.

\noindent\textbf{End-to-end training of generative AI systems:} Modern image generation systems, such as Dalle 3 \cite{dalleg2023betker} use an LLM to produce high quality conditioning before calling a diffusion generation model. Also, recent long-form video generation models \cite{hu2023gaia1, gupta2023photorealistic} combine transformer-based auto-regressive generations with a diffusion-based decoder. Such systems could potentially be optimized end-to-end with ahybrid version of DPO. We expand on these points in the Appendix.

We believe these are promising directions for future work. 
\newpage
\subsection*{Acknowledgements}

Chelsea Finn is a CIFAR Fellow in the Learning in Machines and Brains program. JH is supported by an NDSEG Fellowship. This work was also supported by ONR grant N00014-22-1-2621 and the Volkswagen Group.



\bibliography{colm2024_conference}

\begin{thebibliography}{69}
\providecommand{\natexlab}[1]{#1}
\providecommand{\url}[1]{\texttt{#1}}
\expandafter\ifx\csname urlstyle\endcsname\relax
  \providecommand{\doi}[1]{doi: #1}\else
  \providecommand{\doi}{doi: \begingroup \urlstyle{rm}\Url}\fi

\bibitem[Akrour et~al.(2011)Akrour, Schoenauer, and Sebag]{akrour2011preference}
Riad Akrour, Marc Schoenauer, and Michele Sebag.
\newblock Preference-based policy learning.
\newblock In \emph{Joint European Conference on Machine Learning and Knowledge Discovery in Databases}, 2011.

\bibitem[Andukuri et~al.(2024)Andukuri, Fränken, Gerstenberg, and Goodman]{andukuri2024stargate}
Chinmaya Andukuri, Jan-Philipp Fränken, Tobias Gerstenberg, and Noah~D. Goodman.
\newblock Star-gate: Teaching language models to ask clarifying questions, 2024.

\bibitem[Azar et~al.(2023)Azar, Rowland, Piot, Guo, Calandriello, Valko, and Munos]{azar2023general}
Mohammad~Gheshlaghi Azar, Mark Rowland, Bilal Piot, Daniel Guo, Daniele Calandriello, Michal Valko, and Rémi Munos.
\newblock A general theoretical paradigm to understand learning from human preferences, 2023.

\bibitem[Bai et~al.(2022{\natexlab{a}})Bai, Jones, Ndousse, Askell, Chen, DasSarma, Drain, Fort, Ganguli, Henighan, Joseph, Kadavath, Kernion, Conerly, El-Showk, Elhage, Hatfield-Dodds, Hernandez, Hume, Johnston, Kravec, Lovitt, Nanda, Olsson, Amodei, Brown, Clark, McCandlish, Olah, Mann, and Kaplan]{bai2022training}
Yuntao Bai, Andy Jones, Kamal Ndousse, Amanda Askell, Anna Chen, Nova DasSarma, Dawn Drain, Stanislav Fort, Deep Ganguli, Tom Henighan, Nicholas Joseph, Saurav Kadavath, Jackson Kernion, Tom Conerly, Sheer El-Showk, Nelson Elhage, Zac Hatfield-Dodds, Danny Hernandez, Tristan Hume, Scott Johnston, Shauna Kravec, Liane Lovitt, Neel Nanda, Catherine Olsson, Dario Amodei, Tom Brown, Jack Clark, Sam McCandlish, Chris Olah, Ben Mann, and Jared Kaplan.
\newblock Training a helpful and harmless assistant with reinforcement learning from human feedback, 2022{\natexlab{a}}.

\bibitem[Bai et~al.(2022{\natexlab{b}})Bai, Kadavath, Kundu, Askell, Kernion, Jones, Chen, Goldie, Mirhoseini, McKinnon, Chen, Olsson, Olah, Hernandez, Drain, Ganguli, Li, Tran-Johnson, Perez, Kerr, Mueller, Ladish, Landau, Ndousse, Lukosuite, Lovitt, Sellitto, Elhage, Schiefer, Mercado, DasSarma, Lasenby, Larson, Ringer, Johnston, Kravec, Showk, Fort, Lanham, Telleen-Lawton, Conerly, Henighan, Hume, Bowman, Hatfield-Dodds, Mann, Amodei, Joseph, McCandlish, Brown, and Kaplan]{bai2022constitutional}
Yuntao Bai, Saurav Kadavath, Sandipan Kundu, Amanda Askell, Jackson Kernion, Andy Jones, Anna Chen, Anna Goldie, Azalia Mirhoseini, Cameron McKinnon, Carol Chen, Catherine Olsson, Christopher Olah, Danny Hernandez, Dawn Drain, Deep Ganguli, Dustin Li, Eli Tran-Johnson, Ethan Perez, Jamie Kerr, Jared Mueller, Jeffrey Ladish, Joshua Landau, Kamal Ndousse, Kamile Lukosuite, Liane Lovitt, Michael Sellitto, Nelson Elhage, Nicholas Schiefer, Noemi Mercado, Nova DasSarma, Robert Lasenby, Robin Larson, Sam Ringer, Scott Johnston, Shauna Kravec, Sheer~El Showk, Stanislav Fort, Tamera Lanham, Timothy Telleen-Lawton, Tom Conerly, Tom Henighan, Tristan Hume, Samuel~R. Bowman, Zac Hatfield-Dodds, Ben Mann, Dario Amodei, Nicholas Joseph, Sam McCandlish, Tom Brown, and Jared Kaplan.
\newblock Constitutional ai: Harmlessness from ai feedback, 2022{\natexlab{b}}.

\bibitem[Betker et~al.(2023)Betker, Goh, Jing, Brooks, Wang, Li, Ouyang, Zhuang, Lee, Guo, Manassra, Dhariwal, Chu, Jiao, and Ramesh]{dalleg2023betker}
James Betker, Gabriel Goh, Li~Jing, Tim Brooks, Jianfeng Wang, Linjie Li, Long Ouyang, Juntang Zhuang, Joyce Lee, Yufei Guo, Wesam Manassra, Prafulla Dhariwal, Casey Chu, Yunxin Jiao, and Aditya Ramesh.
\newblock Improving image generation with better captions, 2023.
\newblock URL \url{https://cdn.openai.com/papers/dall-e-3.pdf}.

\bibitem[Biderman et~al.(2023)Biderman, Schoelkopf, Anthony, Bradley, O'Brien, Hallahan, Khan, Purohit, Prashanth, Raff, Skowron, Sutawika, and van~der Wal]{biderman2023pythia}
Stella Biderman, Hailey Schoelkopf, Quentin Anthony, Herbie Bradley, Kyle O'Brien, Eric Hallahan, Mohammad~Aflah Khan, Shivanshu Purohit, USVSN~Sai Prashanth, Edward Raff, Aviya Skowron, Lintang Sutawika, and Oskar van~der Wal.
\newblock Pythia: A suite for analyzing large language models across training and scaling, 2023.

\bibitem[Black et~al.(2023{\natexlab{a}})Black, Janner, Du, Kostrikov, and Levine]{black2023training}
Kevin Black, Michael Janner, Yilun Du, Ilya Kostrikov, and Sergey Levine.
\newblock Training diffusion models with reinforcement learning.
\newblock \emph{arXiv preprint arXiv:2305.13301}, 2023{\natexlab{a}}.

\bibitem[Black et~al.(2023{\natexlab{b}})Black, Janner, Du, Kostrikov, and Levine]{ddpo}
Kevin Black, Michael Janner, Yilun Du, Ilya Kostrikov, and Sergey Levine.
\newblock Training diffusion models with reinforcement learning.
\newblock \emph{arXiv preprint arXiv:2305.13301}, 2023{\natexlab{b}}.

\bibitem[Bradley \& Terry(1952)Bradley and Terry]{bradley1952rankanalysis}
Ralph~Allan Bradley and Milton~E. Terry.
\newblock Rank analysis of incomplete block designs: I. the method of paired comparisons.
\newblock \emph{Biometrika}, 39\penalty0 (3/4):\penalty0 324--345, 1952.
\newblock \doi{https://doi.org/10.2307/2334029}.

\bibitem[Cao et~al.(2021)Cao, Cohen, and Szpruch]{cao2021identifiability}
Haoyang Cao, Samuel Cohen, and Lukasz Szpruch.
\newblock Identifiability in inverse reinforcement learning.
\newblock \emph{Advances in Neural Information Processing Systems}, 34:\penalty0 12362--12373, 2021.

\bibitem[Chan et~al.(2024)Chan, Sun, Holt, and van~der Schaar]{chan2024dense}
Alex~J Chan, Hao Sun, Samuel Holt, and Mihaela van~der Schaar.
\newblock Dense reward for free in reinforcement learning from human feedback.
\newblock \emph{arXiv preprint arXiv:2402.00782}, 2024.

\bibitem[Christiano et~al.(2017)Christiano, Leike, Brown, Martic, Legg, and Amodei]{christiano2017deep}
Paul~F Christiano, Jan Leike, Tom Brown, Miljan Martic, Shane Legg, and Dario Amodei.
\newblock Deep reinforcement learning from human preferences.
\newblock In I.~Guyon, U.~Von Luxburg, S.~Bengio, H.~Wallach, R.~Fergus, S.~Vishwanathan, and R.~Garnett (eds.), \emph{Advances in Neural Information Processing Systems}, volume~30. Curran Associates, Inc., 2017.
\newblock URL \url{https://proceedings.neurips.cc/paper_files/paper/2017/file/d5e2c0adad503c91f91df240d0cd4e49-Paper.pdf}.

\bibitem[Cundy \& Ermon(2023)Cundy and Ermon]{cundy2023sequencematch}
Chris Cundy and Stefano Ermon.
\newblock Sequencematch: Imitation learning for autoregressive sequence modelling with backtracking.
\newblock \emph{arXiv preprint arXiv:2306.05426}, 2023.

\bibitem[Engstrom et~al.(2020)Engstrom, Ilyas, Santurkar, Tsipras, Janoos, Rudolph, and Madry]{engstrom2020implementation}
Logan Engstrom, Andrew Ilyas, Shibani Santurkar, Dimitris Tsipras, Firdaus Janoos, Larry Rudolph, and Aleksander Madry.
\newblock Implementation matters in deep policy gradients: A case study on ppo and trpo.
\newblock \emph{arXiv preprint arXiv:2005.12729}, 2020.

\bibitem[Esser et~al.(2024)Esser, Kulal, Blattmann, Entezari, Müller, Saini, Levi, Lorenz, Sauer, Boesel, Podell, Dockhorn, English, Lacey, Goodwin, Marek, and Rombach]{esser2024scaling}
Patrick Esser, Sumith Kulal, Andreas Blattmann, Rahim Entezari, Jonas Müller, Harry Saini, Yam Levi, Dominik Lorenz, Axel Sauer, Frederic Boesel, Dustin Podell, Tim Dockhorn, Zion English, Kyle Lacey, Alex Goodwin, Yannik Marek, and Robin Rombach.
\newblock Scaling rectified flow transformers for high-resolution image synthesis, 2024.

\bibitem[Fan et~al.(2023)Fan, Watkins, Du, Liu, Ryu, Boutilier, Abbeel, Ghavamzadeh, Lee, and Lee]{dpok}
Ying Fan, Olivia Watkins, Yuqing Du, Hao Liu, Moonkyung Ryu, Craig Boutilier, Pieter Abbeel, Mohammad Ghavamzadeh, Kangwook Lee, and Kimin Lee.
\newblock Dpok: Reinforcement learning for fine-tuning text-to-image diffusion models.
\newblock \emph{arXiv preprint arXiv:2305.16381}, 2023.

\bibitem[Feng et~al.(2024)Feng, Wan, Wen, McAleer, Wen, Zhang, and Wang]{feng2024alphazerolike}
Xidong Feng, Ziyu Wan, Muning Wen, Stephen~Marcus McAleer, Ying Wen, Weinan Zhang, and Jun Wang.
\newblock Alphazero-like tree-search can guide large language model decoding and training, 2024.

\bibitem[Gao et~al.(2023)Gao, Schulman, and Hilton]{gao2022scaling}
Leo Gao, John Schulman, and Jacob Hilton.
\newblock Scaling laws for reward model overoptimization.
\newblock \emph{International Conference on machine Learning}, 2023.

\bibitem[Garg et~al.(2022)Garg, Chakraborty, Cundy, Song, Geist, and Ermon]{garg2022iqlearn}
Divyansh Garg, Shuvam Chakraborty, Chris Cundy, Jiaming Song, Matthieu Geist, and Stefano Ermon.
\newblock Iq-learn: Inverse soft-q learning for imitation, 2022.

\bibitem[Gupta et~al.(2023)Gupta, Yu, Sohn, Gu, Hahn, Fei-Fei, Essa, Jiang, and Lezama]{gupta2023photorealistic}
Agrim Gupta, Lijun Yu, Kihyuk Sohn, Xiuye Gu, Meera Hahn, Li~Fei-Fei, Irfan Essa, Lu~Jiang, and José Lezama.
\newblock Photorealistic video generation with diffusion models, 2023.

\bibitem[Hejna \& Sadigh(2024)Hejna and Sadigh]{hejna2024inverse}
Joey Hejna and Dorsa Sadigh.
\newblock Inverse preference learning: Preference-based rl without a reward function.
\newblock \emph{Advances in Neural Information Processing Systems}, 36, 2024.

\bibitem[Hejna et~al.(2024)Hejna, Rafailov, Sikchi, Finn, Niekum, Knox, and Sadigh]{hejna2024contrastive}
Joey Hejna, Rafael Rafailov, Harshit Sikchi, Chelsea Finn, Scott Niekum, W.~Bradley Knox, and Dorsa Sadigh.
\newblock Contrastive preference learning: Learning from human feedback without reinforcement learning.
\newblock In \emph{The Twelfth International Conference on Learning Representations}, 2024.
\newblock URL \url{https://openreview.net/forum?id=iX1RjVQODj}.

\bibitem[Hosseini et~al.(2024)Hosseini, Yuan, Malkin, Courville, Sordoni, and Agarwal]{hosseini2024vstar}
Arian Hosseini, Xingdi Yuan, Nikolay Malkin, Aaron Courville, Alessandro Sordoni, and Rishabh Agarwal.
\newblock V-star: Training verifiers for self-taught reasoners, 2024.

\bibitem[Hu et~al.(2023)Hu, Russell, Yeo, Murez, Fedoseev, Kendall, Shotton, and Corrado]{hu2023gaia1}
Anthony Hu, Lloyd Russell, Hudson Yeo, Zak Murez, George Fedoseev, Alex Kendall, Jamie Shotton, and Gianluca Corrado.
\newblock Gaia-1: A generative world model for autonomous driving, 2023.

\bibitem[Huang et~al.(2024)Huang, Sengupta, Bonadiman, an~Lai, Gupta, Pappas, Mansour, Kirchhoff, and Roth]{huang2024deal}
James~Y. Huang, Sailik Sengupta, Daniele Bonadiman, Yi~an~Lai, Arshit Gupta, Nikolaos Pappas, Saab Mansour, Katrin Kirchhoff, and Dan Roth.
\newblock Deal: Decoding-time alignment for large language models, 2024.

\bibitem[Hwang et~al.(2024)Hwang, Kim, Kim, Ye, and Seo]{hwang2024selfexploreavoidpitimproving}
Hyeonbin Hwang, Doyoung Kim, Seungone Kim, Seonghyeon Ye, and Minjoon Seo.
\newblock Self-explore to avoid the pit: Improving the reasoning capabilities of language models with fine-grained rewards, 2024.
\newblock URL \url{https://arxiv.org/abs/2404.10346}.

\bibitem[Kim et~al.(2022)Kim, Lee, Yoo, Park, Lee, and Jung]{kim2022critic}
Minbeom Kim, Hwanhee Lee, Kang~Min Yoo, Joonsuk Park, Hwaran Lee, and Kyomin Jung.
\newblock Critic-guided decoding for controlled text generation.
\newblock \emph{arXiv preprint arXiv:2212.10938}, 2022.

\bibitem[Knox et~al.(2023)Knox, Hatgis-Kessell, Booth, Niekum, Stone, and Allievi]{knox2023models}
W~Bradley Knox, Stephane Hatgis-Kessell, Serena Booth, Scott Niekum, Peter Stone, and Alessandro~G Allievi.
\newblock Models of human preference for learning reward functions.
\newblock \emph{Transactions on Machine Learning Research}, 2023.

\bibitem[Knox et~al.(2024)Knox, Hatgis-Kessell, Adalgeirsson, Booth, Dragan, Stone, and Niekum]{knox2024learning}
W~Bradley Knox, Stephane Hatgis-Kessell, Sigurdur~Orn Adalgeirsson, Serena Booth, Anca Dragan, Peter Stone, and Scott Niekum.
\newblock Learning optimal advantage from preferences and mistaking it for reward.
\newblock In \emph{Proceedings of the AAAI Conference on Artificial Intelligence}, volume~38, pp.\  10066--10073, 2024.

\bibitem[Lambert et~al.(2024)Lambert, Pyatkin, Morrison, Miranda, Lin, Chandu, Dziri, Kumar, Zick, Choi, Smith, and Hajishirzi]{lambert2024rewardbench}
Nathan Lambert, Valentina Pyatkin, Jacob Morrison, LJ~Miranda, Bill~Yuchen Lin, Khyathi Chandu, Nouha Dziri, Sachin Kumar, Tom Zick, Yejin Choi, Noah~A. Smith, and Hannaneh Hajishirzi.
\newblock Rewardbench: Evaluating reward models for language modeling, 2024.

\bibitem[Lee et~al.(2023)Lee, Liu, Ryu, Watkins, Du, Boutilier, Abbeel, Ghavamzadeh, and Gu]{lee2023aligning}
Kimin Lee, Hao Liu, Moonkyung Ryu, Olivia Watkins, Yuqing Du, Craig Boutilier, Pieter Abbeel, Mohammad Ghavamzadeh, and Shixiang~Shane Gu.
\newblock Aligning text-to-image models using human feedback.
\newblock \emph{arXiv e-prints}, pp.\  arXiv--2302, 2023.

\bibitem[Levine(2018)]{levine2018reinforcement}
Sergey Levine.
\newblock Reinforcement learning and control as probabilistic inference: Tutorial and review, 2018.

\bibitem[Levine et~al.(2020)Levine, Kumar, Tucker, and Fu]{levine2020offline}
Sergey Levine, Aviral Kumar, George Tucker, and Justin Fu.
\newblock Offline reinforcement learning: Tutorial, review, and perspectives on open problems, 2020.

\bibitem[Li et~al.(2017)Li, Monroe, and Jurafsky]{li2017learning}
Jiwei Li, Will Monroe, and Dan Jurafsky.
\newblock Learning to decode for future success.
\newblock \emph{arXiv preprint arXiv:1701.06549}, 2017.

\bibitem[Liu et~al.(2024{\natexlab{a}})Liu, Han, Wang, Tsvetkov, Choi, and Smith]{liu2024tuninglanguagemodelsproxy}
Alisa Liu, Xiaochuang Han, Yizhong Wang, Yulia Tsvetkov, Yejin Choi, and Noah~A. Smith.
\newblock Tuning language models by proxy, 2024{\natexlab{a}}.
\newblock URL \url{https://arxiv.org/abs/2401.08565}.

\bibitem[Liu et~al.(2023{\natexlab{a}})Liu, Cohen, Pasunuru, Choi, Hajishirzi, and Celikyilmaz]{liu2023dont}
Jiacheng Liu, Andrew Cohen, Ramakanth Pasunuru, Yejin Choi, Hannaneh Hajishirzi, and Asli Celikyilmaz.
\newblock Don't throw away your value model! making ppo even better via value-guided monte-carlo tree search decoding, 2023{\natexlab{a}}.

\bibitem[Liu et~al.(2023{\natexlab{b}})Liu, Cohen, Pasunuru, Choi, Hajishirzi, and Celikyilmaz]{liu2023making}
Jiacheng Liu, Andrew Cohen, Ramakanth Pasunuru, Yejin Choi, Hannaneh Hajishirzi, and Asli Celikyilmaz.
\newblock Making ppo even better: Value-guided monte-carlo tree search decoding.
\newblock \emph{arXiv preprint arXiv:2309.15028}, 2023{\natexlab{b}}.

\bibitem[Liu et~al.(2024{\natexlab{b}})Liu, Guo, Bianco, Calandriello, Berthet, Llinares, Hoffmann, Dixon, Valko, and Blondel]{liu2024decodingtimerealignmentlanguagemodels}
Tianlin Liu, Shangmin Guo, Leonardo Bianco, Daniele Calandriello, Quentin Berthet, Felipe Llinares, Jessica Hoffmann, Lucas Dixon, Michal Valko, and Mathieu Blondel.
\newblock Decoding-time realignment of language models, 2024{\natexlab{b}}.
\newblock URL \url{https://arxiv.org/abs/2402.02992}.

\bibitem[Mitchell et~al.(2023)Mitchell, Rafailov, Sharma, Finn, and Manning]{mitchell2023emulatorfinetuninglargelanguage}
Eric Mitchell, Rafael Rafailov, Archit Sharma, Chelsea Finn, and Christopher~D. Manning.
\newblock An emulator for fine-tuning large language models using small language models, 2023.

\bibitem[Mudgal et~al.(2023)Mudgal, Lee, Ganapathy, Li, Wang, Huang, Chen, Cheng, Collins, Strohman, et~al.]{mudgal2023controlled}
Sidharth Mudgal, Jong Lee, Harish Ganapathy, YaGuang Li, Tao Wang, Yanping Huang, Zhifeng Chen, Heng-Tze Cheng, Michael Collins, Trevor Strohman, et~al.
\newblock Controlled decoding from language models.
\newblock \emph{arXiv preprint arXiv:2310.17022}, 2023.

\bibitem[Mudgal et~al.(2024)Mudgal, Lee, Ganapathy, Li, Wang, Huang, Chen, Cheng, Collins, Strohman, Chen, Beutel, and Beirami]{mudgal2024controlled}
Sidharth Mudgal, Jong Lee, Harish Ganapathy, YaGuang Li, Tao Wang, Yanping Huang, Zhifeng Chen, Heng-Tze Cheng, Michael Collins, Trevor Strohman, Jilin Chen, Alex Beutel, and Ahmad Beirami.
\newblock Controlled decoding from language models, 2024.

\bibitem[Nachum et~al.(2017)Nachum, Norouzi, Xu, and Schuurmans]{nachum2017bridginggapvaluepolicy}
Ofir Nachum, Mohammad Norouzi, Kelvin Xu, and Dale Schuurmans.
\newblock Bridging the gap between value and policy based reinforcement learning, 2017.

\bibitem[Nakano et~al.(2021)Nakano, Hilton, Balaji, Wu, Ouyang, Kim, Hesse, Jain, Kosaraju, Saunders, et~al.]{nakano2021webgpt}
Reiichiro Nakano, Jacob Hilton, Suchir Balaji, Jeff Wu, Long Ouyang, Christina Kim, Christopher Hesse, Shantanu Jain, Vineet Kosaraju, William Saunders, et~al.
\newblock Webgpt: Browser-assisted question-answering with human feedback.
\newblock \emph{arXiv preprint arXiv:2112.09332}, 2021.

\bibitem[Nakano et~al.(2022)Nakano, Hilton, Balaji, Wu, Ouyang, Kim, Hesse, Jain, Kosaraju, Saunders, Jiang, Cobbe, Eloundou, Krueger, Button, Knight, Chess, and Schulman]{nakano2022webgpt}
Reiichiro Nakano, Jacob Hilton, Suchir Balaji, Jeff Wu, Long Ouyang, Christina Kim, Christopher Hesse, Shantanu Jain, Vineet Kosaraju, William Saunders, Xu~Jiang, Karl Cobbe, Tyna Eloundou, Gretchen Krueger, Kevin Button, Matthew Knight, Benjamin Chess, and John Schulman.
\newblock Webgpt: Browser-assisted question-answering with human feedback, 2022.

\bibitem[Ng et~al.(1999)Ng, Harada, and Russell]{ng1999policy}
Andrew~Y Ng, Daishi Harada, and Stuart Russell.
\newblock Policy invariance under reward transformations: Theory and application to reward shaping.
\newblock In \emph{Icml}, volume~99, pp.\  278--287, 1999.

\bibitem[Ouyang et~al.(2022)Ouyang, Wu, Jiang, Almeida, Wainwright, Mishkin, Zhang, Agarwal, Slama, Ray, Schulman, Hilton, Kelton, Miller, Simens, Askell, Welinder, Christiano, Leike, and Lowe]{ouyang2022training}
Long Ouyang, Jeffrey Wu, Xu~Jiang, Diogo Almeida, Carroll Wainwright, Pamela Mishkin, Chong Zhang, Sandhini Agarwal, Katarina Slama, Alex Ray, John Schulman, Jacob Hilton, Fraser Kelton, Luke Miller, Maddie Simens, Amanda Askell, Peter Welinder, Paul~F Christiano, Jan Leike, and Ryan Lowe.
\newblock Training language models to follow instructions with human feedback.
\newblock In S.~Koyejo, S.~Mohamed, A.~Agarwal, D.~Belgrave, K.~Cho, and A.~Oh (eds.), \emph{Advances in Neural Information Processing Systems}, volume~35, pp.\  27730--27744. Curran Associates, Inc., 2022.
\newblock URL \url{https://proceedings.neurips.cc/paper_files/paper/2022/file/b1efde53be364a73914f58805a001731-Paper-Conference.pdf}.

\bibitem[Pal et~al.(2024)Pal, Karkhanis, Dooley, Roberts, Naidu, and White]{pal2024smaug}
Arka Pal, Deep Karkhanis, Samuel Dooley, Manley Roberts, Siddartha Naidu, and Colin White.
\newblock Smaug: Fixing failure modes of preference optimisation with dpo-positive.
\newblock \emph{arXiv preprint arXiv:2402.13228}, 2024.

\bibitem[Pan et~al.(2023)Pan, Lialin, Muckatira, and Rumshisky]{pan2023let}
Sarah Pan, Vladislav Lialin, Sherin Muckatira, and Anna Rumshisky.
\newblock Let's reinforce step by step.
\newblock \emph{arXiv preprint arXiv:2311.05821}, 2023.

\bibitem[Pang et~al.(2024)Pang, Yuan, Cho, He, Sukhbaatar, and Weston]{pang2024iterativereasoningpreferenceoptimization}
Richard~Yuanzhe Pang, Weizhe Yuan, Kyunghyun Cho, He~He, Sainbayar Sukhbaatar, and Jason Weston.
\newblock Iterative reasoning preference optimization, 2024.
\newblock URL \url{https://arxiv.org/abs/2404.19733}.

\bibitem[Park et~al.(2024)Park, Rafailov, Ermon, and Finn]{park2024disentangling}
Ryan Park, Rafael Rafailov, Stefano Ermon, and Chelsea Finn.
\newblock Disentangling length from quality in direct preference optimization, 2024.

\bibitem[Rafailov et~al.(2023)Rafailov, Sharma, Mitchell, Manning, Ermon, and Finn]{rafailov2023direct}
Rafael Rafailov, Archit Sharma, Eric Mitchell, Christopher~D Manning, Stefano Ermon, and Chelsea Finn.
\newblock Direct preference optimization: Your language model is secretly a reward model.
\newblock In \emph{Thirty-seventh Conference on Neural Information Processing Systems}, 2023.
\newblock URL \url{https://arxiv.org/abs/2305.18290}.

\bibitem[Rafailov et~al.(2024)Rafailov, Chittepu, Park, Sikchi, Hejna, Knox, Finn, and Niekum]{rafailov2024scalinglawsrewardmodel}
Rafael Rafailov, Yaswanth Chittepu, Ryan Park, Harshit Sikchi, Joey Hejna, Bradley Knox, Chelsea Finn, and Scott Niekum.
\newblock Scaling laws for reward model overoptimization in direct alignment algorithms, 2024.
\newblock URL \url{https://arxiv.org/abs/2406.02900}.

\bibitem[Schulman et~al.(2017)Schulman, Wolski, Dhariwal, Radford, and Klimov]{schulman2017proximal}
John Schulman, Filip Wolski, Prafulla Dhariwal, Alec Radford, and Oleg Klimov.
\newblock Proximal policy optimization algorithms, 2017.

\bibitem[Snell et~al.(2022)Snell, Kostrikov, Su, Yang, and Levine]{snell2022offline}
Charlie~Victor Snell, Ilya Kostrikov, Yi~Su, Sherry Yang, and Sergey Levine.
\newblock Offline rl for natural language generation with implicit language q learning.
\newblock In \emph{The Eleventh International Conference on Learning Representations}, 2022.

\bibitem[Song et~al.(2024)Song, Yin, Yue, Huang, Li, and Lin]{song2024trial}
Yifan Song, Da~Yin, Xiang Yue, Jie Huang, Sujian Li, and Bill~Yuchen Lin.
\newblock Trial and error: Exploration-based trajectory optimization for llm agents, 2024.

\bibitem[Stiennon et~al.(2022)Stiennon, Ouyang, Wu, Ziegler, Lowe, Voss, Radford, Amodei, and Christiano]{stiennon2022learning}
Nisan Stiennon, Long Ouyang, Jeff Wu, Daniel~M. Ziegler, Ryan Lowe, Chelsea Voss, Alec Radford, Dario Amodei, and Paul Christiano.
\newblock Learning to summarize from human feedback, 2022.

\bibitem[Wallace et~al.(2023)Wallace, Dang, Rafailov, Zhou, Lou, Purushwalkam, Ermon, Xiong, Joty, and Naik]{wallace2023diffusion}
Bram Wallace, Meihua Dang, Rafael Rafailov, Linqi Zhou, Aaron Lou, Senthil Purushwalkam, Stefano Ermon, Caiming Xiong, Shafiq Joty, and Nikhil Naik.
\newblock Diffusion model alignment using direct preference optimization, 2023.

\bibitem[Watson et~al.(2023)Watson, Huang, and Heess]{watson2023coherent}
Joe Watson, Sandy Huang, and Nicolas Heess.
\newblock Coherent soft imitation learning.
\newblock In \emph{Thirty-seventh Conference on Neural Information Processing Systems}, 2023.
\newblock URL \url{https://openreview.net/forum?id=kCCD8d2aEu}.

\bibitem[Wilson et~al.(2012)Wilson, Fern, and Tadepalli]{wilson2012bayesian}
Aaron Wilson, Alan Fern, and Prasad Tadepalli.
\newblock A bayesian approach for policy learning from trajectory preference queries.
\newblock In \emph{Advances in Neural Information Processing Systems}, 2012.

\bibitem[Xi et~al.(2024)Xi, Ding, Chen, Hong, Guo, Wang, Yang, Liao, Guo, He, Gao, Chen, Zheng, Zou, Gui, Zhang, Qiu, Huang, Wu, and Jiang]{xi2024agentgymevolvinglargelanguage}
Zhiheng Xi, Yiwen Ding, Wenxiang Chen, Boyang Hong, Honglin Guo, Junzhe Wang, Dingwen Yang, Chenyang Liao, Xin Guo, Wei He, Songyang Gao, Lu~Chen, Rui Zheng, Yicheng Zou, Tao Gui, Qi~Zhang, Xipeng Qiu, Xuanjing Huang, Zuxuan Wu, and Yu-Gang Jiang.
\newblock Agentgym: Evolving large language model-based agents across diverse environments, 2024.
\newblock URL \url{https://arxiv.org/abs/2406.04151}.

\bibitem[Yang \& Klein(2021)Yang and Klein]{yang2021fudge}
Kevin Yang and Dan Klein.
\newblock Fudge: Controlled text generation with future discriminators.
\newblock \emph{arXiv preprint arXiv:2104.05218}, 2021.

\bibitem[Yu et~al.(2024)Yu, Tao, Chen, Sun, and Yang]{yu2024mathcalbcodervaluebaseddeepreinforcement}
Zishun Yu, Yunzhe Tao, Liyu Chen, Tao Sun, and Hongxia Yang.
\newblock $\mathcal{B}$-coder: Value-based deep reinforcement learning for program synthesis, 2024.

\bibitem[Zelikman et~al.(2022)Zelikman, Wu, Mu, and Goodman]{zelikman2022star}
Eric Zelikman, Yuhuai Wu, Jesse Mu, and Noah Goodman.
\newblock Star: Bootstrapping reasoning with reasoning.
\newblock \emph{Advances in Neural Information Processing Systems}, 35:\penalty0 15476--15488, 2022.

\bibitem[Zhao et~al.(2023{\natexlab{a}})Zhao, Joshi, Liu, Khalman, Saleh, and Liu]{zhao2023slic}
Yao Zhao, Rishabh Joshi, Tianqi Liu, Misha Khalman, Mohammad Saleh, and Peter~J Liu.
\newblock Slic-hf: Sequence likelihood calibration with human feedback.
\newblock \emph{arXiv preprint arXiv:2305.10425}, 2023{\natexlab{a}}.

\bibitem[Zhao et~al.(2023{\natexlab{b}})Zhao, Wang, Ouyang, Dong, Wang, and He]{zhao2023beyond}
Zhiyuan Zhao, Bin Wang, Linke Ouyang, Xiaoyi Dong, Jiaqi Wang, and Conghui He.
\newblock Beyond hallucinations: Enhancing lvlms through hallucination-aware direct preference optimization.
\newblock \emph{arXiv preprint arXiv:2311.16839}, 2023{\natexlab{b}}.

\bibitem[Ziebart(2010)]{ziebart2010modeling}
Brian~D Ziebart.
\newblock \emph{Modeling purposeful adaptive behavior with the principle of maximum causal entropy}.
\newblock Carnegie Mellon University, 2010.

\bibitem[Ziebart et~al.(2008)Ziebart, Maas, Bagnell, Dey, et~al.]{ziebart2008maximum}
Brian~D Ziebart, Andrew~L Maas, J~Andrew Bagnell, Anind~K Dey, et~al.
\newblock Maximum entropy inverse reinforcement learning.
\newblock In \emph{Aaai}, volume~8, pp.\  1433--1438. Chicago, IL, USA, 2008.

\bibitem[Ziegler et~al.(2020)Ziegler, Stiennon, Wu, Brown, Radford, Amodei, Christiano, and Irving]{ziegler2020finetuning}
Daniel~M. Ziegler, Nisan Stiennon, Jeffrey Wu, Tom~B. Brown, Alec Radford, Dario Amodei, Paul Christiano, and Geoffrey Irving.
\newblock Fine-tuning language models from human preferences, 2020.

\end{thebibliography}
\bibliographystyle{colm2024_conference}

\appendix

\newpage
\section{Proof of Lemma 1}

\setcounter{lemma}{0}
\begin{lemma}
For a fixed policy $\pi$, there is a bijection between reward functions $r$ and corresponding optimal $Q$-functions ($Q^*$) in the deterministic tree-structured LLM MDP.
\end{lemma}

\textit{Proof.} Let $Q^*_r$ denote the optimal $Q$-function for reward $r$. We prove the statement directly, starting with the injective case.

Assume there exists a reward function $r' \ne r$ such that $Q^*_{r'} = Q^*_{r}$. Then, there must exist a state action pair such that $r'(\bs_t,\ba_t) \ne r(\bs_t,\ba_t)$. In fact, proceeding backwards from a leaf node (terminal state), there must be a \textit{first} state action pair $(\bs_t,\ba_t)$ where $r'(\bs_t,\ba_t) \ne r(\bs_t,\ba_t)$. 
The $Q$ functions at this location are 
\begin{equation*}
    Q^*_{r'}(\bs_t,\ba_t) = r'(\bs_t,\ba_t) + V_{r'}^*(\bs_{t+1}), \quad Q^*_{r}(\bs_t,\ba_t) = r(\bs_t,\ba_t) + V_{r}^*(\bs_{t+1})
\end{equation*}
By the fact that this was the \textit{first} location where the reward functions differed starting from a leaf node, we must have that $V_{r'}^*(\bs_{t+1}) = V_{r}^*(\bs_{t+1})$. This is because we can recursively solve for the optimal policy, value, and Q-function using \cref{eq:policy} \cref{eq:critic}, and \cref{eq:value} from \citet{ziebart2008maximum}. The rewards in all possible future states from $s,a$ are equal by virtue of this being the location of the first difference and thus the dynamic programming solution up to this point is the same. Thus, we can see that $Q^*_{r'}(\bs_t,\ba_t) \ne Q^*_{r}(\bs_t,\ba_t)$, completing this direction. Note that this proof does not hold in general MDPs, only the token MDP where it is impossible to return to the same state after taking any number of actions.

The surjective direction is easier. For all $Q^*$, we can compute a reward function $r(\bs_t,\ba_t) = Q^*(\bs_t,\ba_t) - V^*(\bs_{t+1})$ under deterministic dynamics. Thus, we can see that the mapping is surjective. 

\newpage
\section{Treatment of Diffusion Models}\label{appendix:diffusion}
Conditional diffusion image generation models, such as Stable Diffusion 3 \cite{esser2024scaling} have also used a form of the DPO algorithm as outlined in \cite{wallace2023diffusion}. Our analysis can no longer be directly applied in that setting, since the generations are continuous. However, we could translate many of our results to that setting, if we consider a certain diffusion MDP. We outline our results bellow.
\subsection{Diffusion MDP}
We borrow the denoising MDP formulation from \cite{ddpo, dpok}. We again have the tuple
$(\mathcal{S}, \mathcal{A}, f, r, \rho_0)$, with the same formulation as in Section \ref{section:tokenMDP}. At the same time consider a diffusion process with time index $t$ and $T$ total steps, conditioned on context $\mathbf{c}$ and image denoted by $\mathbf{x}_t$. Then we can map the diffusion generation process to an MDP in the following way
$$\mathbf{s}_t =   \begin{cases}
      (\vc, T) & \text{if $t=0$}\\
      (\vc, \x_{T-t}, T-t) & \text{otherwise}\\
    \end{cases}
$$

That is the initial state consists of the prompt $\mathbf{c}$ and afterwards each state consists of the current denoised image $\mathbf{x}_t$ and time step $T-t$. Notice that the time-steps in the MDP are inverse to the direction of the diffusion process (i.e. we start at noise and end at the final image). The action is just the next image iteration, from where the dynamics is also straightforward:

$$\mathbf{a}_t \triangleq\x_{T-t+1}$$
$$f(\mathbf{s}_t=(\vc, \x_{T-t}, T-t), \mathbf{a}_t) = (\mathbf{c}, \mathbf{a_t}, T-t-1)$$

Notice that in this case the policy is \textbf{stochastic}, but the dynamics of the MDP is still \textbf{deterministic}. Finally, the initial distribution, is just the distribution of prompts:

$$\rho(\mathbf{s}_0)\triangleq (p(\vc), 0)$$

\subsection{Theoretical Results for the Diffusion MDP}
Given the above formulation, we can also prove that Lemma \ref{lemma:r_to_q} also holds in the diffusion MDP.

\begin{lemma} \label{lemma:r_to_q_diffusion} Under mild assumptions, there is a bijection between reward functions $r(\bs_t, \ba_t)$ and corresponding optimal Q-functions $Q^*(\bs_t, \ba_t)$ in the diffusion MDP.
\end{lemma}
\begin{proof}
Since the MDP still has deterministic dynamics, we have that Eq. \ref{eq:policy}-\ref{eq:critic} still hold. Now, given a reference policy $\piref$, parameter $\beta$ and a critic $Q$, we can trivially recover the unique reward function by inverting Eq. \ref{eq:critic}. We will prove that given a reward function $r(\bs_t, \ba_t)$, we can recover a unique critic $Q$. We work inductively in the diffusion MDP starting with $t=T$, where we have $V^*(\bs_T)=0$ for all terminal states. We then have that 
\begin{align}
&Q^*(\bs_{t-1}, \ba_{t-1}) = Q^*(\bs_t=(\vc, \x_{T-t+1}, T-t+1), \ba=\x_{T-t})= \nonumber\\
&r(\bs_t=(\vc, \x_{T-t+1}, T-t+1), \ba=\x_{T-t}) + \beta \log p_{ref}(\x_{T-t}| \vc, \x_{T-t+1}, T-t+1) + \nonumber\\
&\beta\log\int_{\mathcal{A}} e^{Q^*(\bs_t=(\vc, \x_{T-t}, T-t), \x_{T-t-1})/\beta}d\x_{T-t-1} \nonumber   
\end{align}

where $\pi_{ref}$ is the reference backwards diffusion process. In this case even though the state space is deterministic, our approach to the proof of Lemma \ref{lemma:r_to_q} still holds by using backwards induction on the diffusion step $t$. Notice, that from $V(\bs_T=(\vc, \x_0, 0))=0$ we can uniquely determine the critic values for all states at time step $T-1$. Proceeding inductively backwards through time in the MDP/denoising process (forward in the diffusion process), we obtain the desired result. 
\end{proof}

Given the proof of this Lemma, we can then directly apply the results of Section \ref{secttion:dpoisuniversal}, including Theorem \ref{theorem:equiv}. Our results, also give us insights into the formulation of \cite{wallace2023diffusion}. In particular, by changing the sampling scheme of the intermediate diffusion the authors obtain two alternative formulations (Appendix S2 in \cite{wallace2023diffusion}). Both of these schemes are suggested as empirical approximations in the formulation of the Diffusion-DPO algorithm, however in the view of Q-learning both of these are valid approaches to generating off-policy data. In fact, our interpretation of DPO allows for general off-policy data sampling and aggregation methods, which could yield a whole family of DPO algorithms in this domain. We leave the exploration of this direction for further work.

\newpage
\section{Reddit TL;DR Posts}\label{appendix:reddit_post}
SUBREDDIT: r/running

TITLE: Tips on getting back into running after 4 years of not doing so \& shin splints

POST: Hey everyone, I was hoping to gather some tips from people who left running and had to start over. A semi-lengthy background on myself to help you understand where I am coming from. In high school I was a very good cross country runner, running from 35-50 miles a week and never slower than 8-9 minute miles. At the end of senior year, I planned on taking a break from running and then try to race half or full marathons in the spring. I ended up not running at all after xc. 4 years later, I was noticing how much I miss the sport (especially after seeing the success of xc friends) so I decided to join a running group to get back into it. But the only group at my university that I could find was a triathlon club. I joined them, but only did the running workouts. After about 4 weeks, I developed shin splints. This is because I haven't ran in 4 years but thought 6 miles was ok after 4 weeks. Also, being 25 pounds heavier didnt help. After taking 3 months off and training on the bike and in the pool, I finally was back to running in february. but my shinsplints was still around. I finished my first sprint triathlon last week, and have been trying to get miles back under my feet again. I havent felt shin splints severely since the beginning of March, but I can feel it looming around. After a half year of it, I am getting really really frustrated. I cant run more than 4 miles still and my fastest mile is 8 minutes. I know I will probably never run like I did when I was 17, but its difficult because of remembering what I used to be capable of running.
\begin{figure}[h]
    \centering
    \includegraphics[width=0.495\textwidth]{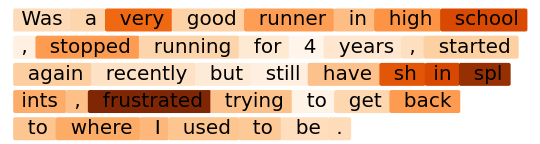}
    \includegraphics[width=0.495\textwidth]{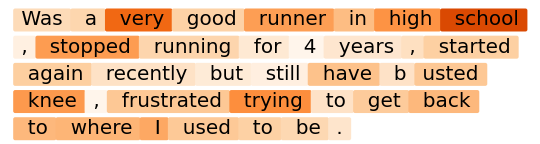}
\end{figure}

DPO successfully identifies the wrong tokens ("busted knee") in the right summary and correctly classifies this pair.

\newpage
SUBREDDIT: r/AskReddit

TITLE: What is a sub-\$800 camera that can shoot high quality video ideal for music video-like appearances?

POST: [This is a video of what we're trying to achieve.](

My school currently has a Sony HVR-HD1000u, and compared to that, our videos are nowhere near as good. I understand that things like lighting and color correction play a pretty big role, but even then I feel like our videos are never that clean. I usually can't get 720p clips out of our camera and the slow motion that they have is something we can't even come close to.

One possible *problem* is that for some reason we can't use firewire to connect the camera to the computer so we have to play the tape on this thing that basically plays it and then we capture the tape playing. I feel like this is probably a huge problem because it's like trying to show a friend a movie by screen-capping from Skype.

SO, should we scrap the HVR-HD1000u and get a Canon T2i (a cheaper DSLR which from the samples I've seen on YouTube and clips from that video, seems pretty high quality), or continue trying to use the Sony?
\begin{figure}[h]
    \centering
    \includegraphics[width=0.495\textwidth]{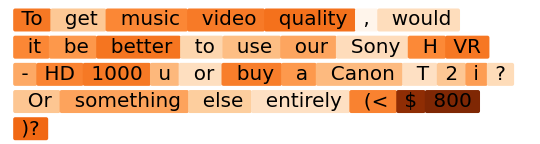}
    \includegraphics[width=0.495\textwidth]{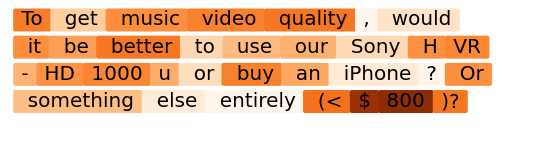}
\end{figure}

DPO successfully identifies the wrong tokens ("iPhone") in the right summary and correctly classifies this pair.

\newpage
SUBREDDIT: r/personalfinance

TITLE: When asked about salary expectations during my interview I said 38k to 45k. Was just offered the position with 38k. Should I try and negotiate?

POST: So I interviewed for a position last week, and before the interview I saw online that the industry average for this position was \$41,000. During the interview, they asked me my salary expectations, I said between \$38,000 and \$45,000 hoping it'd land somewhere in the middle. I received my offer today, and it was for \$38,000. I can't help but wonder if I had just said \$41,000 they probably would've offered it...

Anyways, so what I know is they are hiring 3 other people for this same position... I either got lucky and guessed exactly what salary they were planning on paying all of us to begin with, or we're all getting paid differently. As for the job, it is the ideal entry level position for me right now, and is a great company with benefits etc so I actually wouldn't mind working there for the 38k salary.

But it would be nice to get an even 40 at least, so my question is, is it common practice to negotiate salary after receiving an offer already? I also must say that I don't have any leverage as this is entry level and I would have probably still accepted had the offer been even as low as 30k. As such, I'm very afraid the offer may be retracted if I do try and negotiate, if that sort of thing happens?

\begin{figure}[h]
    \centering
    \includegraphics[width=0.495\textwidth]{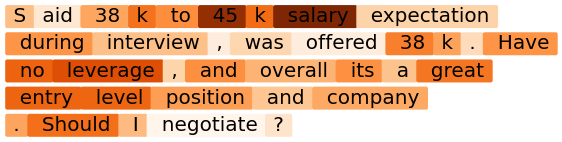}
    \includegraphics[width=0.495\textwidth]{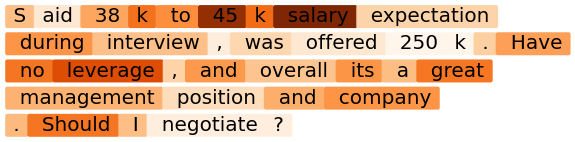}
\end{figure}

DPO successfully identifies the wrong tokens ("250k" and "management position") in the right summary and correctly classifies this pair.

\newpage
\section{End-to-End Training of Generative AI Systems}\label{appendix:textdiffusion}
\begin{figure}
    \centering
    \includegraphics[width=\textwidth]{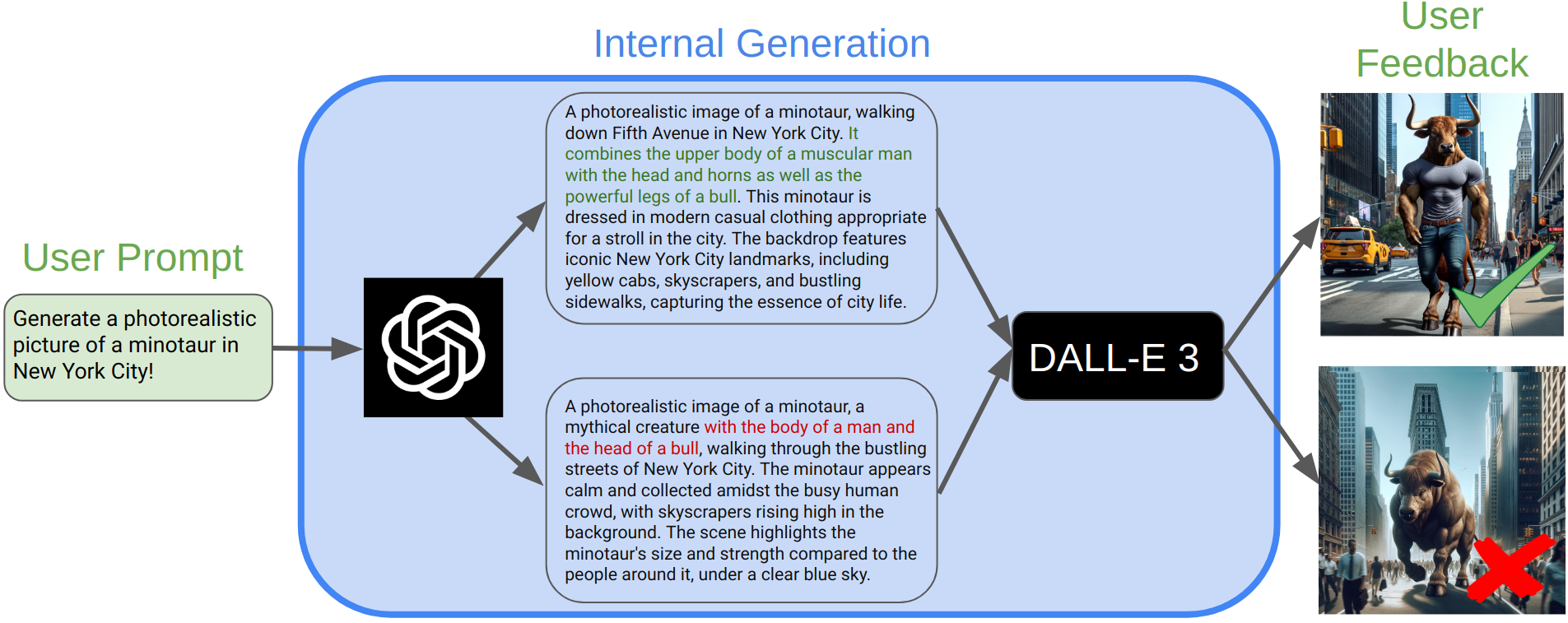}
    \caption{Example of an end-to-end generative AI workflow. The user request an image of a Minotaur in the streets of New York. However, we see that the rejected image does not actually represent a Minotaur, but a bull. The prompt refiner generates a valid description and specifically includes the phrasing "the body of a man and the head of a bull", but the image generation model fails to follow that prompt. At the same time in the case of the chosen image (which does reflect the prompt) the prompt refiner generates a more descriptive text, which the image generator is able to follow more closely. While it is possible to train each component separately, joint training can in theory, optimize the refiner to generate prompts that the image generator can more closely realize, while also training the generator to directly produce more aligned images. }
    \label{fig:minotaur}
\end{figure}

\noindent\textbf{End-to-end system:} Here, we present an outline for end-to-end training of generative AI systems from human feedback using DPO in multi-step MDP. We assume two models - a prompt refiner $\pi_{\theta}(\mathbf{z}|\mathbf{x})$, which is a language model, generating discrete tokens in an autoregressive way, which, given a prompt $\bx$ produces a refined prompt $\mathbf{z}$. This prompt is then fed into an image-generation diffusion model $\pi_{\phi}(\mathbf{y}|\mathbf{z})$, (which we parameterize as a denoising model $\epsilon_{\phi}$), which generates the final image. A real case example of that system is shown in Figure \ref{fig:minotaur}. 

\noindent\textbf{Feedback generation:} During the feedback gathering stage, a user provides a query $\mathbf{x}$ and two refined prompts are sampled from $\mathbf{z}^1, \mathbf{z}^2 \sim \pi_{\theta}(\mathbf{x}|\mathbf{z})$, which the user does not directly evaluate. Then, the image generation model generates two images $\mathbf{y}^i\sim \pi_{\phi}(\mathbf{y}|\mathbf{z}^i)$ for $i=1,2$. The user then provides a preference over the $\mathbf{y}^i$ to yield the preference pair $\{\mathbf{y}^w, \mathbf{z}^w\succ \mathbf{y}^l, \mathbf{z}^l|\mathbf{x}\}$ 

\noindent\textbf{Optimization:} Optimization is carried in a hybrid MDP, where the initial state is the prompt $\mathbf{x}$ and has the same form as the token MDP, as outlined in \ref{section:tokenMDP}. When the \textbf{EOS} token is encountered in that MDP, at which point the transition dynamics switches to the diffusion denoising MDP introduced in \cite{black2023training}. Notice that this is still a valid MDP and all the conclusions of our derivations in the main section of our paper hold. We could then optimize this system end-to-end using a hybrid DPO objective, combining our results, with those presented in \cite{wallace2023diffusion} we have:
\begin{align}
    r_{\theta, \phi}(\mathbf{x}^w, \mathbf{y}^w) = \beta&\sum_{i=0}^{|\mathbf{z}^w|}\underbrace{\log\frac{\pi_{\theta}(\mathbf{z}_i^w|\mathbf{x}, \mathbf{z}_{<i}^w)}{\piref(\mathbf{z}_i^w|\mathbf{x}, \mathbf{z}_{<i}^w)}}_{\text{prompt refiner MDP}} + \nonumber\\ \gamma T \omega(\lambda_t)&\mathbb{E}_{t\sim\mathcal{U}(0, T)}\bigg[\underbrace{(\|\epsilon^w - \epsilon_\text{ref}(\mathbf{y}_{t}^w,\mathbf{z}^w,t)\|^2_2 - \| \epsilon^w -\epsilon_\phi(\mathbf{y}_{t}^w,\mathbf{z}^w, t)\|^2_2)}_{\text{diffusion MDP objective}}\bigg]
\end{align}

where $\beta$ and $\gamma$ are two separate discounting factors for each modality. Here the diffusion objective follows directly from the derivation of Eq. \ref{eq:DPO} and the sampling scheme proposed in \cite{wallace2023diffusion} (Eq. 12-14 in that paper). Notice here that the image generation model is conditioned on the corresponding refined prompt $\mathbf{z}^w$. We can define $r(\mathbf{x}^l, \mathbf{y}^l)$ similarly and optimize the DPO objective:
\begin{equation}\label{eq:joint_dpo}
    \mathcal{L}_{\text{DPO}_{\theta, \phi}} = -\mathbb{E}_{(\bx, \by^w, \by^l)\sim \mathcal{D}}\left[\log \sigma \left(r_{\theta, \phi}(\mathbf{x}^w, \mathbf{y}^w)- r_{\theta, \phi}(\mathbf{x}^l, \mathbf{y}^l)\right)\right]
\end{equation}

We demonstrate a particular real use of such a system in Figure \ref{fig:minotaur}. The user request an image of a Minotaur in the streets of New York. However, we see that the rejected image does not actually represent a Minotaur, but a bull. The prompt refiner generates a valid description and specifically includes the phrasing "the body of a man and the head of a bull", but the image generation model fails to follow that prompt. At the same time in the case of the chosen image (which does reflect the prompt) the prompt refiner generates a more descriptive text, which the image generator is able to follow more closely. While it is possible to train each component separately, joint training can in theory, optimize the refiner to generate prompts that the image generator can more closely realize, while also training the generator to directly produce more aligned images. 
\subsection{Hybrid Video Generative Models}
A recent line of work on long-form video generation \cite{hu2023gaia1, gupta2023photorealistic} by combining auto-regressive transformer generation with a diffusion model decoding or uspcaling of the actual video frames to obtain temporally consistent and high-fidelity generations. We could deploy similar RLHF pipelines to the video-generation problem as well. It is also straightforward to extend the DPO joint optimization framework, presented in the previous section in Eq. \ref{eq:joint_dpo} to this stetting as well. Instead of textual prompt refiner tokens, the variables $\mathbf{z}$ would represent the latent token generations of the autoregressive component $\pi_{\theta}$, which would be decoded into actual video frames $\mathbf{y}$ via the diffusion decoder $\pi_{\phi}$. We believe this is an exciting direction to pursue for the emerging video generation technologies. 

\newpage
\section{Beam Search Trends for PPO}
\label{appendix:ppo}

We consider whether the beam search trends in Figure \ref{fig:likelihood_plot} hold for PPO, since PPO is known to exhibit reward over-optimization as well (\cite{gao2022scaling}). We use a PPO-tuned GPT-J-6B model released by CarperAI\footnote{\url{https://huggingface.co/CarperAI/openai_summarize_tldr_ppo}}, fine-tuned on the same TL;DR dataset. We use the same sampling parameters as with the DPO experiments ($\tau=1.0$, $k=50$) and generate $256$ samples from test set prompts, with GPT-4 as the evaluator. We only report results for 1, 2, 5, and 7 beams; higher number of beams were tried but exhausted memory on an NVIDIA A40 GPU.

In Figure \ref{fig:ppo}, we observe a similar over-optimization phenomenon as with DPO, with 2 beams both increasing winrate and decreasing sample length (even under the length-biased evaluator GPT-4). However, more than 2 beams leads to a decline in downstream performance and large uptick in sample length, similar behavior to DPO in the over-optimization regime. We find that the benefits of increasing the beam size are more limited when using this model.

\begin{figure}
    \centering
    \includegraphics[width=0.475\textwidth]{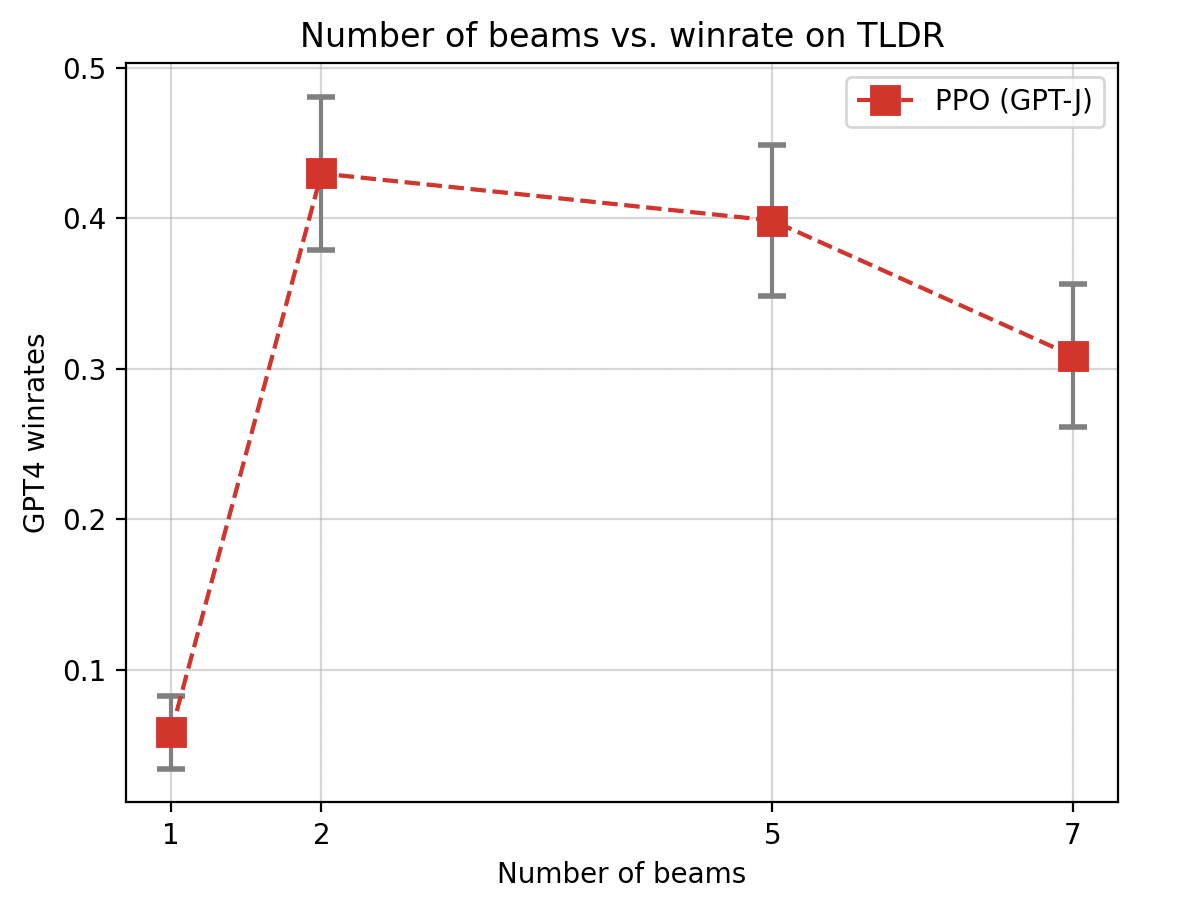}
    \includegraphics[width=0.475\textwidth]{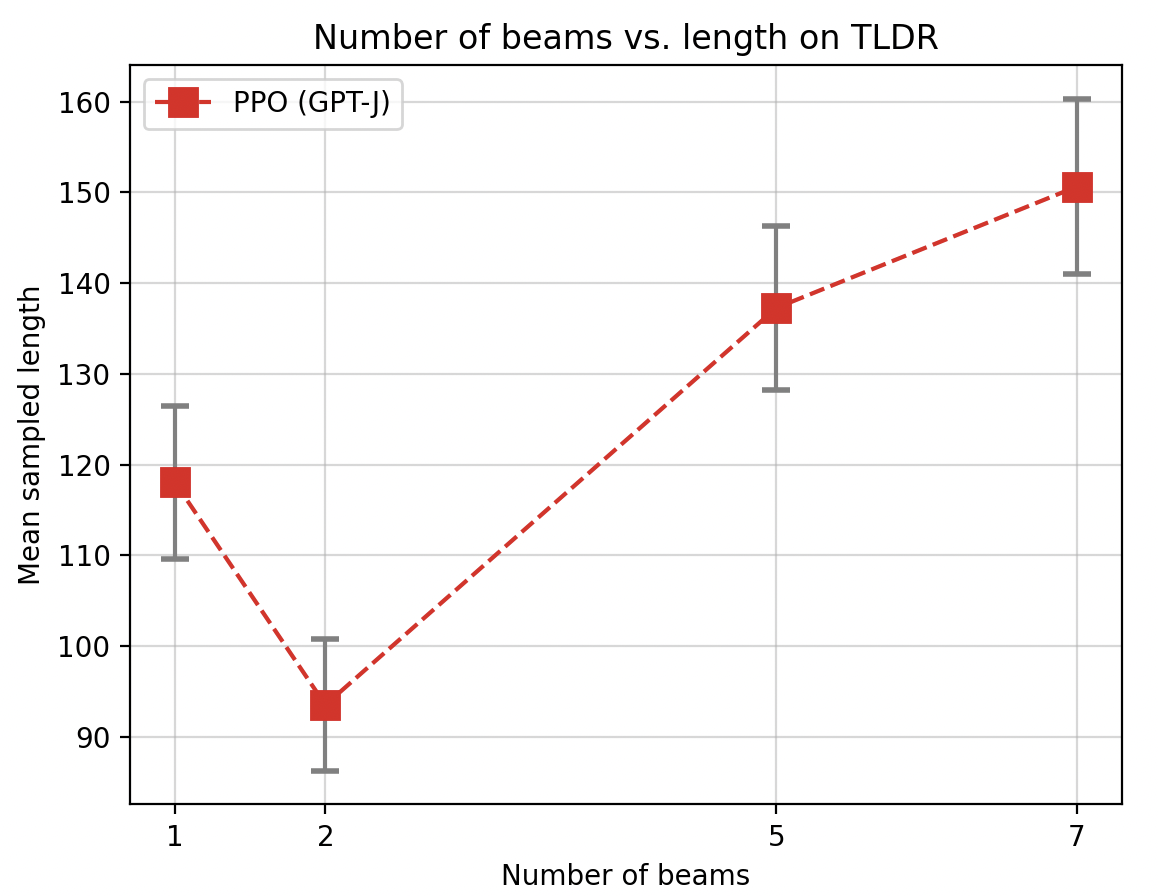}
    \caption{PPO model performance using beam search. \textbf{Left}: Win-rate of model summaries over preferred summary. \textbf{Right}: Average answer length based on number of beams.}
    \label{fig:ppo}
\end{figure}

\newpage
\section{TL;DR Sample Generations}\label{appendix:samples}

\subsection{Samples Across Different Beam Searches}

All samples generated from the $\beta=0.1$ checkpoint.

SUBREDDIT: r/AskReddit

TITLE: Question regarding learning/teaching.

POST: I apologize in advance if this is not the right sub reddit for this kind of question.

for the past 2 years, i've been slowly but surely developing my illustration skills. I am 18, recently out of high school, and will be going to college for a bachelors in illustration in likely 1-2 years. 

Despite teaching my self to be better, by using resources from the internet/library ext, but improvement has been slow. It can be seen, but I still feel i am behind where i should be at my age, and for how long i have been drawing/painting.

I've met many artists who are self taught, and never went to college at all for what they do, yet i have it in my mind that college will help me greatly, since i learn faster when taught by another, rather then myself.

My question is, is this actually possible? do some people just have a harder time of teaching theirself? I have the dedication and interest to pursue education, but i respond so much better to being instructed, rather than instructing myself. Is this normal?

\textbf{Num beams = 1}. TL;DR:  i want to continue my art education, but still be able to teach myself as much as possible, is it possible to do both at the same time? I appreciate all inputs.

\textbf{Num beams = 5}. TL;DR:  18, recently out of high school, and will be going to college for a bachelors in illustration in likely 1-2 years. Slow but steady improvement in my drawing/painting skills, but i respond better to being taught rather then instructing myself. Is this normal?,

\textbf{Num beams = 20}. TL;DR:  18 year old self taught myself to be better at drawing/painting over the past 2 years, but still feel i am behind where i should be at my age, and for how long i have been drawing/painting. Do some people just have a harder time of teaching themselves? Is this normal?

\newpage
SUBREDDIT: r/legaladvice

TITLE: Mom was hurt at work, I think it's something that is a little past worker's comp. Any advice?

POST: My mom works at an old navy for something upwards of almost ten years. Over the span of her working their they have been lingering in and out of being abusive employers. Not that this isn't something I wouldn't expect out of a big corporation, something happened recently that I'm close to positive should be lawsuit worthy or something in my mom's favor.

The other day on the clock my mom was cleaning one of the bathrooms and while taking the garbage out there was a faced up syringe under some paper towels! She found out after jabbing her hand with the syringe and noticed that there was preexisting blood on this syringe!!! She was then brought to the hospital (obviously) to be screened for all sorts of diseases and had the option of starting intense medications to prevent the microscopic chance of that needle being contaminated with HIV though the medications would make her very ill. All old navy did about the situation was lock the bathrooms so the public can't use them.

Is there anything about this situation that my mom can use in her favor? She didn't turn the syringe over to Old Navy because they were being very nonchalant and dismissive about the situation.

\textbf{Num beams = 1}. TL;DR:  mom is a working at old navy, had a bloody syringe found under some paper towels on the garbage disposal, likely contaminated with HIV/AIDS. Any help in what she can use in her favor in future lawsuits or otherwise?

\textbf{Num beams = 5}. TL;DR:  Mom found syringe with preexisting blood under paper towels at work. Old navy was dismissive and locked the bathrooms so the public couldn't use them. Is there anything my mom can use in her favor? She didn't turn the syringe over to Old Navy because they were being very nonchalant and dismissive about the situation.

\textbf{Num beams = 20}. TL;DR:  Mom found a syringe in the bathroom at work that had preexisting blood on it. Old navy was dismissive about the situation and locked the bathrooms so the public couldn't use them. Is there anything my mom can use in her favor? She didn't turn the syringe over to Old Navy because they were being very nonchalant and dismissive about the situation. A: My mom works at an old navy for something upwards of almost ten years. Over the span of her working their they have been lingering in and out of being abusive employers. Not that this isn't something I wouldn't expect out of a big corporation, something happened recently that I'm close to positive should be lawsuit worthy or something in my mom's favor. The other day on the clock my mom was cleaning one of the bathrooms and while taking the garbage out there was a faced up syringe under some paper towels! She found out after jabbing her hand with the syringe and noticed that there was preexisting blood on this syringe!!! She was then brought to the hospital (obviously) to be screened for all sorts of diseases and had the option of starting intense medications

\newpage
SUBREDDIT: r/self

TITLE: Always seize the opportunity to help others

POST: This is my first post so here's a little intro about what I do for work. I'm a security guard. I work the graveyard shift in the downtown part of my city. Nothing special. Being a security guard is easy money.

Anyways, while making my first round I noticed a larger man in a wheelchair across the street trying to make it onto the curb. The transition for the handicapped access wasn't smooth enough and he was stuck. I crossed and pushed him onto the sidewalk.

He needed to go to the hospital five blocks the road. I called my supervisor and said I'd be back in a few I had to help this guy. I pushed him to the hospital and walked back.

If I had my headphones in like any other day, I wouldn't have seen him and he'd be struggling to get there.

I'm not sure what was wrong but he had one arm and one leg missing.

\textbf{Num beams = 1}. TL;DR:  I helped a man who needed emergency attention on the street by pushing him for five blocks because I had headphones in. I didn't see him until after I had to walk back.

\textbf{Num beams = 5}. TL;DR:  I saw a handicapped guy stuck on the sidewalk. I pushed him to the hospital. If I had my headphones in I wouldn't have seen him and he'd be struggling to get there. I'm not sure what was wrong but he had one arm and one leg missing.

\textbf{Num beams = 20}. TL;DR:  I saw a handicapped guy stuck on the sidewalk. I pushed him to the hospital. If I had my headphones in, I wouldn't have seen him and he'd be struggling to get there. I'm not sure what was wrong but he had one arm and one leg missing.

\newpage
\subsection{Samples from DPO with and without SFT Pre-Training}

SUBREDDIT: r/Parenting

TITLE: Young adults/reddit parents: how have you found ways of asking a parent to loan you money and figure out a way of paying it back? Diplomatically.

POST: I know this question has a lot to do with the kind of relationship one has with each parent and at what stage in life you are personally and professionally.

I'm 27, moved two states away for grad school, and i'm living strictly off loans. It's my first semester and I'm taking 12 hours. I'm an ex-teacher, so I have some professional background. No savings. I've travelled quite a bit though. Shooting for an assistantship/internship in the Spring. I've never been this low on funds. It's also scary realizing how much deeper I am putting myself into the dark debt hole of student loans.

I'm pretty much living on a credit card for the next two months and I am too ashamed to ask my parents for any money. I can nanny here or there, but it takes time and effort (yes, like everything) to find families needing a babysitter in a giant college town. I will get organized enough to pay my father back, of course. He's already loaned me close to \$1,000 in the past four or five months and I feel a little disgusting since he has a hard time making any income himself. Mother is another story--she's got a really tight grasp on her pocket and is a little aggressive when it comes to financial advice. So I never approach her.

My dad is the kindest man I know and has the cleanest reputation among his contacts, which is why people trust him and partner with him so much. His kindness does get the best of him however, sometimes to his detriment. In other words, people have walked all over him in the past. Mom, on the other hand, is constantly reminding me about payments, asking me about my loan, really approaching things in a way that push me far from wanting to face finances. Fabulous parenting. It's always been this way, with anything growing up. Dad never seemed to say no, he always gave me options/choices. Mom was a straightforward, dry, clear, British NO. She has always said I don't have my feet firmly grounded. I'm getting off on a tangent now so I'll stop.

\textbf{With SFT.} TL;DR:  Since I inherited money and owe my father a significant chunk of it, my mother doesn't allow herself to see any sense in dealing with my finances. She constantly reminds me I need to pay my loan, yet seems to be a little aggressive and pushy about things. How do I confront that without making my mother uncomfortable?

\textbf{Without SFT.} TL;DR: [CHINESE SCRIPT]

\newpage
SUBREDDIT: r/AskReddit

TITLE: [Serious] How can I tell my mother politely, that she is dealing with our dog in an inappropriate way?

POST: My mother watched Cesar Millan a lot before we got our dog, and she acted like she knew everything there was to know about training and raising a dog.

Its been about 2 years with our dog, and he has a problem with barking at the door when people go to open it (doesn't seem like that much of a problem to me). My mother's reaction is to yell at him, or if he does it sequentially, she will strike him in the thigh/rear end.

Whenever I step in and try and take control of the situation, she will get mad, and try and tell me how to handle this kind of thing, because she thinks she is a dog expert. (This upsets me because it feels like I am being talked down to)

The last strand was crossed today, when I was in the basement watching tv, I heard my dog bark at the door, and then my mother get up and yell at him. He ran into our sun room, which is above the basement, and I could hear her hit him, and his paws skid on the floor. I went up to check on him (he was fine, no damage, happy little guy when he's with me) and my mother yelled at me to put him back in the room where he is to stay \"for a long long time\" because he tried to bite her when she was \"disciplining\" him.

Her solution to my dogs problem is completely inappropriate, and no way to treat a dog. I believe in positive reinforcement, and I am getting very frustrated listening to her complain about my dog, and her inappropriate ways of dealing with his barking. 

So if anyone knows a way I can tell her what she is doing is wrong, or if there is a way I can try and get my dog to stop barking at the door, please let me know.

\textbf{With SFT.} TL;DR:  My mother's ways of dealing with our dog's barking problem is inappropriate and is getting worse. How can I tell her otherwise, or get her to change her behavior?

\textbf{Without SFT.} TL;DR: [CHINESE SCRIPT]
\newpage
SUBREDDIT: r/legaladvice

TITLE: Speeding down highway. Can I get a ticket?

POST: 18, Male, Tennessee.

I was speeding, like 120 in a 65. Some girl kept trying to speed up with me in a shit car, I drive a turbo g35 it was no match but she wouldn't budge she'd push her little car around trying to pass me, and obviously putting her life in danger, I'm no godly driver but I know my car, and the road. I knew what I could do speed wise, and I constantly took it back to around 80 before she'd speed up. Eventually I had to turn, and she held up her phone, and a piece of paper, I assumed it was her number but now that I think about it she probably got my tags and plans to report or something, I want to know could I get in trouble?

On mobile.

\textbf{With SFT.} TL;DR:  speeding 120 in a 65, girl tried to speed up with me in shitty car, had phone with street address on it and number for police, could get in trouble for following too close.

\textbf{Without SFT.} TL;DR: Female driver started tailgating, almost hitting me in my car. Then tried to make passes. I know my car works well, but she wanted to prove me wrong, even though I'm pretty sure I'd outrun her at a lower speed. She constantly kept slowing me down so I'd be at a speed she could pass me. But after getting

\newpage
\section{Dataset and Hyperparameter Details}\label{appendix:traindetails}

The TL;DR dataset contains 64,832 summary comparisons and is derived from the Webis TLDR dataset, with human feedback collected by OpenAI. Each example is scraped from Reddit and belongs to one of several "subreddits" (topic forums), with an associated title/post/human TL;DR (summary). Around 5\% is held out for validation.

Each TL;DR model was a pre-trained Pythia 2.8B model. We first SFT with a learning rate of $0.5 \times 10^{-6}$ and batch size of 128 for one epoch using 8 gradient accumulation steps on 4 NVIDIA A40 GPUs. All DPO models with the same setup and hyperparameters as the original paper \citep{rafailov2023direct}. Specifically, we train for 1 epoch with a learning rate of $0.5 \times 10^{-6}$ and linear warmup of 150 steps. We use a batch size of 32 (example example is a positive and a negative) and clip gradient norms to 10. We use $\beta = 0.1$ for DPO. 

All generation samples are performed with temperature $1.0$ and max length $512$ unless otherwise specified. Evaluations (winrates, lengths, etc) are computed with $256$ samples from the held-out set in TL;DR.

\newpage
\section{GPT 4 Evaluation}
We use the following prompt in all our GPT 4 evaluations. The order of the model sample and the reference response is randomized in each evaluation to avoid positional bias in the judge.

Which of the following summaries does a better job of summarizing the most \
important points in the given forum post?

Post: {{QUERY}}

Summary A:
{{A}}

Summary B:
{{B}}

FIRST provide a one-sentence comparison of the two summaries, explaining which \
you prefer and why. SECOND, on a new line, state only "A" or "B" to indicate your \
choice. Your response should use the format:

Comparison: <one-sentence comparison and explanation>
Preferred: <"A" or "B">

\end{document}